\documentclass{article}

     \usepackage[nonatbib, final]{neurips_2020}

\usepackage[utf8]{inputenc} 
\usepackage[T1]{fontenc}    

\usepackage{url}            
\usepackage{booktabs}       
\usepackage{amsfonts}       
\usepackage{nicefrac}       
\usepackage{microtype}      
\usepackage{hyperref}       

\usepackage{amsmath,amsthm,amssymb,bm}
\usepackage{soul,color}
\usepackage{graphicx,float,wrapfig}
\usepackage{subcaption}
\usepackage{wrapfig}
\usepackage{ifpdf}
\usepackage{listings}
\usepackage{diagbox}

\usepackage{algorithm}
\usepackage{algorithmicx}
\usepackage{algpseudocode}

\DeclareMathOperator*{\argmax}{arg\,max}
\DeclareMathOperator*{\argmin}{arg\,min}

\newcommand{\x}{\mathbf{x}}
\newcommand{\y}{\mathbf{y}}
\newcommand{\w}{\mathbf{w}}

\newcommand{\z}{\mathbf{z}}

\newcommand{\bb}{\mathbf{b}}
\newcommand{\R}{\mathbb{R}}
\newcommand{\zero}{\mathbf{0}}
\newcommand{\hessian}{\mathbf{H}}
\newcommand{\jacobian}{\mathbf{J}}
\newcommand{\iden}{\mathbf{I}}

\newcommand{\ma}{\mathbf{A}}
\newcommand{\mb}{\mathbf{B}}
\newcommand{\mc}{\mathbf{C}}
\newcommand{\matp}{\mathbf{P}}
\newcommand{\mg}{\mathbf{G}}

\newcommand{\todo}{{\color{red} \textbf{TODO}}}

\definecolor{mydarkblue}{rgb}{0,0.1,0.6}
\definecolor{mydarkred}{rgb}{0.6,0.1,0}

\newtheorem{manualtheoreminner}{Theorem}
\newenvironment{manualtheorem}[1]{%
  \renewcommand\themanualtheoreminner{#1}%
  \manualtheoreminner
}{\endmanualtheoreminner}

\newtheorem{theorem}{Theorem}
\newtheorem{lemma}{Lemma}
\newtheorem{corollary}{Corollary}
\theoremstyle{definition}
\newtheorem{definition}{Definition}
\newtheorem{assumption}{Assumption}
\newtheorem{fact}{Fact}

\title{Improved Algorithms for Convex-Concave \\Minimax Optimization}

\author{%
  Yuanhao Wang\\
  Computer Science Department\\
  Princeton University\\
  \texttt{yuanhao@princeton.edu} \\
  \And
  Jian Li\\
  Institute for Interdisciplinary Information Sciences\\
  Tsinghua University\\
  \texttt{lijian83@mail.tsinghua.edu.cn} \\
}

\begin{document}

\maketitle

\begin{abstract}
This paper studies minimax optimization problems $\min_\x \max_\y f(\x,\y)$, where $f(\x,\y)$ is $m_\x$-strongly convex with respect to $\x$, $m_\y$-strongly concave with respect to $\y$ and $(L_\x,L_{\x\y},L_\y)$-smooth. Zhang et al. \cite{zhang2019lower} 
provided the following lower bound of the gradient complexity for
any first-order method:
$\Omega\Bigl(\sqrt{\frac{L_\x}{m_\x}+\frac{L_{\x\y}^2}{m_\x m_\y}+\frac{L_\y}{m_\y}}\ln(1/\epsilon)\Bigr).$ 
This paper proposes a new algorithm with gradient complexity upper bound 
$\Tilde{O}\Bigl(\sqrt{\frac{L_\x}{m_\x}+\frac{L\cdot L_{\x\y}}{m_\x m_\y}+\frac{L_\y}{m_\y}}\ln\left(1/\epsilon\right)\Bigr),$ where $L=\max\{L_\x,L_{\x\y},L_\y\}$. This improves over the best known upper bound $\Tilde{O}\left(\sqrt{\nicefrac{L^2}{m_\x m_\y}} \ln^3\left(1/\epsilon\right)\right)$
by Lin et al. \cite{lin2020near}. Our bound achieves linear convergence rate and tighter dependency on condition numbers, especially when 
$L_{\x\y}\ll L$ (i.e., when the interaction between $\x$ and $\y$ is weak).
Via reduction, our new bound also implies improved bounds for strongly convex-concave and convex-concave minimax optimization problems. 
When $f$ is quadratic, we can further improve the upper bound, 
which matches the lower bound up to a small sub-polynomial factor. 
\end{abstract}

\section{Introduction}
In this paper, we study the following minimax optimization problem
\begin{equation}
\label{equ:minimaxform}
	\min_{\x\in\R^{n}} \max_{\y\in\R^m} f(\x,\y).
\end{equation}
This problem can be thought as finding the equilibrium in a zero-sum two-player game, and has been studied extensively in game theory, economics and computer science. This formulation also arises in many machine learning applications, including adversarial training~\cite{madry2017towards,sinha2017certifiable}, prediction and regression problems~\cite{xu2005maximum,taskar2006structured}, reinforcement learning~\cite{du2017stochastic,dai2017sbeed,nachum2019dualdice} and generative adversarial networks~\cite{goodfellow2014generative,arjovsky2017wasserstein}.

We study the fundamental setting where $f$ is smooth, strongly convex w.r.t. $\x$ and strongly concave w.r.t. $\y$. In particular, we consider the function class $\mathcal{F}(m_\x,m_\y,L_\x,L_{\x\y},L_\y)$, where $m_\x$ is the strong convexity modulus, $m_\y$ is the strong concavity modulus, $L_\x$ and $L_\y$ characterize the smoothness w.r.t. $\x$ and $\y$ respectively, and $L_{\x\y}$ characterizes the interaction between $\x$ and $\y$ (see Definition~\ref{def:func}). The reason to consider such a function class is twofold. First, the strongly convex-strongly concave setting is fundamental. Via reduction~\cite{lin2020near}, an efficient algorithm for this setting implies efficient algorithms for other settings, including strongly convex-concave, convex-concave, and non-convex-concave settings. Second, Zhang et al. \cite{zhang2019lower} recently proved a gradient complexity lower bound
$\Omega\Bigl(\sqrt{\frac{L_\x}{m_\x}+\frac{L_{\x\y}^2}{m_\x m_\y}+\frac{L_\y}{m_\y}}\cdot\ln\left(\frac{1}{\epsilon}\right)\Bigr)$,
which naturally depends on the above parameters.~\footnote{This lower bound is also proved by Ibrahim et al. \cite{ibrahim2019linear}. Although their result is stated for a narrower class of algorithms, their proof actually works for the broader class of algorithms considered in~\cite{zhang2019lower}.}

\begin{figure}[t]
	\centering
    \includegraphics[width=0.6\textwidth]{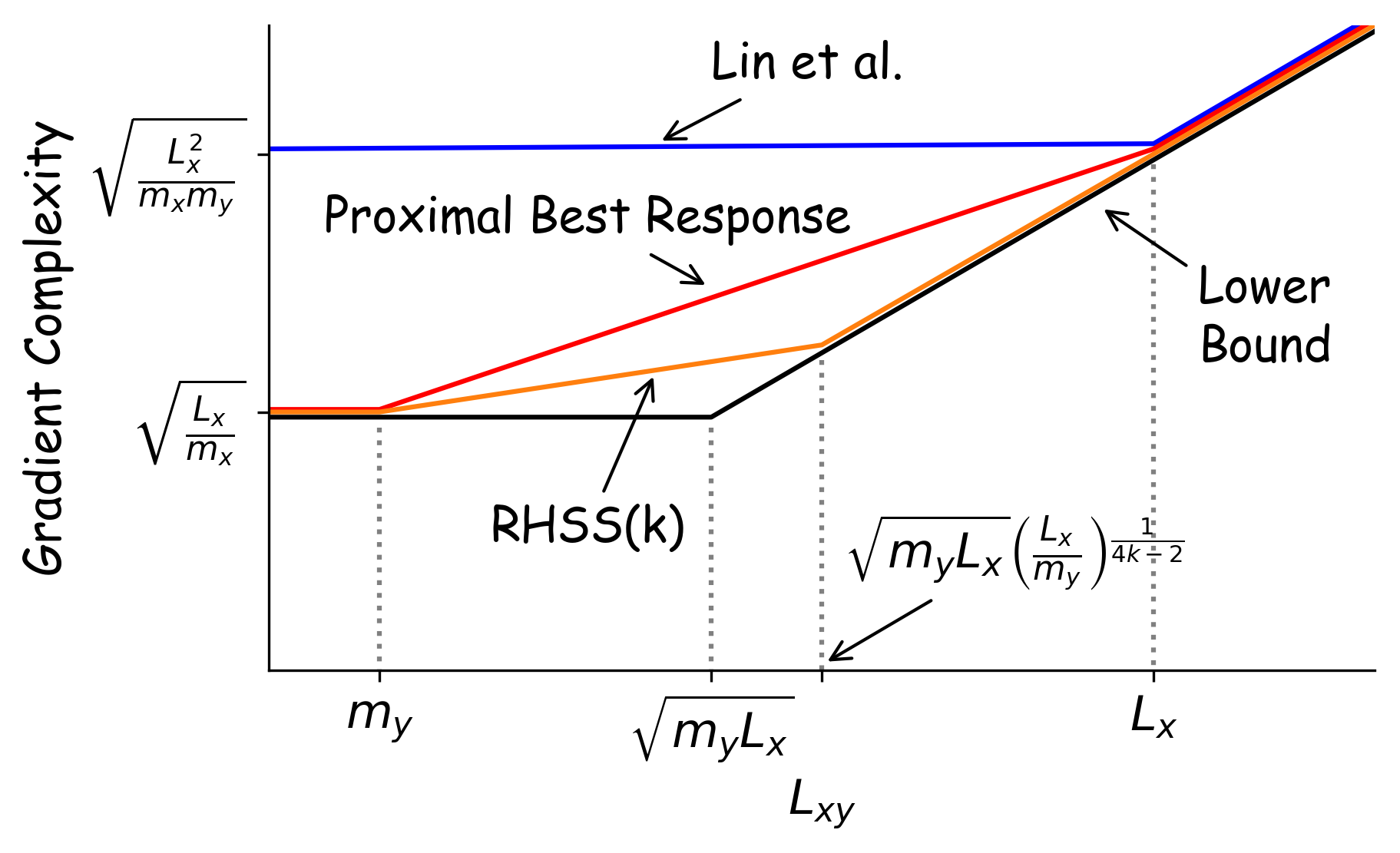}
    \caption{Comparison of previous upper bound~\cite{lin2020near}, lower bound~\cite{zhang2019lower} and the results in this paper when $L_\x=L_\y$, $m_\x < m_\y$, ignoring logarithmic factors. The upper bounds and lower bounds are shown as a function of $L_{\x\y}$ while other parameters are fixed.}
    \vspace{-0.2cm}
	\label{fig:comparison}
\end{figure}

In this setting, classic algorithms such as Gradient Descent-Ascent and ExtraGradient~\cite{korpelevich1976extragradient} can achieve linear convergence~\cite{tseng1995linear,zhang2019lower}; however, their dependence on the condition number is far from optimal. 
Recently, Lin et al.~\cite{lin2020near} showed an upper bound of $\tilde{O}\left(\sqrt{L^2/m_\x m_\y}\ln^3(1/\epsilon)\right)$, which has a much tighter dependence on the condition number. In particular, when $L_{\x\y}>\max\{L_\x,L_\y\}$, the dependence on the condition number matches the lower bound. However, when $L_{\x\y}\ll \max\{L_\x,L_\y\}$, this dependence would no longer be tight (see Fig~\ref{fig:comparison} for illustration). 
In particular, we note that, when $x$ and $y$ are completely decoupled (i.e., $L_{\x\y}=0$), the optimal gradient complexity bound is $\Theta\left(\sqrt{L_\x/m_\x+L_\y/m_\y}\cdot\ln\left(1/\epsilon\right)\right)$ (the upper bound can be obtained by simply optimizing $\x$ and $\y$ separately). 
Moreover, Lin et al.'s result does not enjoy a linear rate, which may be undesirable if a high precision solution is needed.

In this work, we propose new algorithms in order to address these two issues. Our contribution can be summarized as follows.
\begin{enumerate}
\item
 For general functions in $\mathcal{F}(m_\x,m_\y,L_\x,L_{\x\y},L_\y)$, we design an algorithm called Proximal Best Response (Algorithm~\ref{algo:pbr}), and prove a convergence rate of $$\tilde{O}\left(\sqrt{\frac{L_\x}{m_\x}+\frac{L_{\x\y}\cdot L}{m_\x m_\y}+\frac{L_\y}{m_\y}}\ln(1/\epsilon)\right).$$ 
 It achieves linear convergence, and has a better dependence on condition numbers when $L_{\x\y}$ is small (see Theorem~\ref{thm:phr} and the red line in Fig.~\ref{fig:comparison}). 
 \item 
 We obtain tighter upper bounds for the strongly-convex concave problem and the general convex-concave problem, by reducing them to 
 the strongly convex-strongly concave problem (See Corollary~\ref{cor:scc} and \ref{cor:scc2}).
\item
We also study the special case where $f$ is a quadratic function. 
We propose an algorithm called Recursive Hermitian-Skew-Hermitian Split (RHSS($k$)), and show that it achieves an upper bound of 
$$
O\left(\sqrt{\frac{L_\x}{m_\x}+\frac{L_{\x\y}^2}{m_\x m_\y}+\frac{L_\y}{m_\y}}\left(\frac{L^2}{m_\x m_\y}\right)^{o(1)}\ln(1/\epsilon)\right).
$$
Details can be found in Theorem~\ref{thm:rhss} and Corollary~\ref{cor:rhss}.
We note that the lower bound by Zhang et al. \cite{zhang2019lower}
holds for quadratic functions as well.
Hence, our upper bound matches the gradient complexity lower bound up to a sub-polynomial factor. 
\end{enumerate}

\section{Preliminaries}
In this work we are interested in strongly-convex strongly-concave smooth problems. 
We first review some standard definitions of strong convexity and smoothness.
A function $f:\R^n\to \R^m$ is $L$-Lipschitz if $\forall \x,\x'\in\R^n$ $\Vert f(\x)-f(\x')\Vert\le L\Vert\x-\x'\Vert.$
	A function $f:\R^n\to \R^m$ is $L$-smooth if $\nabla f$ is $L$-Lipschitz.
A differentiable function $\phi:\R^n\to \R$ is said to be $m$-strongly convex if for any $\x,\x'\in\R^n$, $\phi(\x')\ge \phi(\x)+(\x'-\x)^T\nabla \phi(\x)+\frac{m}{2}\Vert \x'-\x\Vert^2$.
If $m=0$, we recover the definition of convexity. If $-\phi$ is $m$-strongly convex, $\phi$ is said to be $m$-strongly concave. For a function $f(\x,\y)$, if $\forall \y$, $f(\cdot,\y)$ is strongly convex, and $\forall \x$, $f(\x,\cdot)$ is strongly concave, then $f$ is said to be strongly convex-strongly concave.

\begin{definition}
\label{def:smooth}
A differentiable function $f:\R^n\times\R^m\to \R$ is said to be $(L_\x,L_{\x\y},L_\y)$-smooth if

\quad\quad 1. For any $\y$, $\nabla_\x f(\cdot,\y)$ is $L_\x$-Lipschitz;
\quad\quad\,	 2. For any $\x$, $\nabla_\y f(\x,\cdot)$ is $L_{\y}$-Lipschitz;

\quad\quad 3. For any $\x$, $\nabla_\x f(\x,\cdot)$ is $L_{\x\y}$-Lipschitz;
\quad\quad	4. For any $\y$, $\nabla_\y f(\cdot,\y)$ is $L_{\x\y}$-Lipschitz.
\end{definition}

In this work, we are interested in function that are strongly convex-strongly concave and smooth. Specifically, we study the following function class.
\begin{definition}
\label{def:func}
The function class $\mathcal{F}\left(m_\x,m_\y,L_\x,L_{\x\y},L_\y\right)$ contains differentiable functions from $\R^n\times\R^m$ to $\R$ such that: 1. $\forall \y$, $f(\cdot,\y)$ is $m_\x$-strongly convex; 2. $\forall \x$, $f(\x,\cdot)$ is $m_\y$-strongly concave; 3. $f$ is $(L_\x,L_{\x\y},L_\y)$-smooth.
\end{definition}

In the case where $f(\x,\y)$ is twice continuously differentiable, denote the Hessian of $f$ at $(\x,\y)$ by $\hessian:=\left[\begin{matrix}
\hessian_{\x\x}&\hessian_{\x\y}\\ \hessian_{\y\x} & \hessian_{\y\y}
\end{matrix}\right]$. Then $\mathcal{F}(m_\x,m_\y,L_\x,L_{\x\y},L_\y)$ can be characterized with the Hessian; in particular we require $m_\x \iden\preccurlyeq \hessian_{\x\x}\preccurlyeq L_\x \iden$, $m_\y\iden\preccurlyeq -\hessian_{\y\y}\preccurlyeq L_\y \iden$ and $\Vert \hessian_{\x\y}\Vert_2\le L_{\x\y}$.


For notational simplicity, we assume that $L_\x=L_\y$ when considering algorithms and upper bounds. This is without loss of generality, since one can define $g(\x,\y):=f((L_\y/L_\x)^{1/4}\x,(L_\x/L_\y)^{1/4}\y)$ in order to make the two smoothness constants equal. It is not hard to show that this rescaling will not change $L_\x/m_\x$, $L_\y/m_\y$, $L_{\x\y}$ and $m_\x m_\y$, and $L=\max\{L_\x,L_{\x\y},L_\y\}$ will not increase. 
Hence, we can make the following assumption without loss of generality.
\footnote{Note that this rescaling also does not change the lower bound.} 
\begin{assumption}
$f\in\mathcal{F}(m_\x,m_\y,L_\x,L_{\x\y},L_\y)$, and $L_\x = L_\y$.
\end{assumption}

The optimal solution of the convex-concave minimax optimization problem $\min_\x \max_\y f(\x,\y)$ is the saddle point $(\x^*,\y^*)$ defined as follows.
\begin{definition}
$(\x^*,\y^*)$ is a saddle point of $f:\R^n\times\R^m\to\R$ if $\forall \x\in\R^n,\y\in\R^m$
$$f(\x,\y^*)\ge f(\x^*,\y^*)\ge f(\x^*,\y).$$ 
\end{definition}

For strongly convex-strongly concave functions, it is well known that such a saddle point exists and is unique. Meanwhile,the saddle point is a stationary point, i.e. $\nabla f(\x^*,\y^*)=0$, and is the minimizer of $\phi(\x):=\max_\y f(\x,\y)$.
For the design of numerical algorithms, we are satisfied with a close enough approximate of the saddle point, called $\epsilon$-saddle points.
\begin{definition}
$(\hat\x,\hat\y)$ is an $\epsilon$-saddle point of $f$ if $\max_\y f(\hat\x,\y)-\min_\x f(\x,\hat\y)\le \epsilon.$
\end{definition}
Alternatively, we can also characterize optimality with the distance to the saddle point. In particular, let $\z^*:=[\x^*;\y^*]$, $\hat\z:=[\hat\x;\hat\y]$, then one may require $\Vert \hat\z-\z^*\Vert\le \epsilon$. This implies that\footnote{See Fact 4 in Appendix~\ref{append:facts} for proof.}
$$\max_\y f(\hat\x,\y)-\min_\x f(\x,\hat\y)\le \frac{L^2}{\min\{m_\x,m_\y\}}\epsilon^2.$$
In this work we focus on first-order methods, that is, algorithms that only access $f$ through gradient evaluations. The complexity of algorithms is measured through the gradient complexity: the number of gradient evaluations required to find an $\epsilon$-saddle point (or get to $\Vert\hat\z-\z^*\Vert\le\epsilon$).

\section{Related Work}

There is a long line of work on the convex-concave saddle point problem. Apart from GDA and ExtraGradient~\cite{korpelevich1976extragradient,tseng1995linear,nemirovski2004prox,gidel2018a}, other algorithms with theoretical guarantees include OGDA~\cite{rakhlin2013online,daskalakis2017training,mokhtari2019unified,azizian2019tight}, Hamiltonian Gradient Descent~\cite{abernethy2019last} and Consensus Optimization~\cite{mescheder2017numerics,abernethy2019last,azizian2019tight}. For the convex-concave case and strongly-convex-concave case, lower bounds have been proven by \cite{ouyang2019lower}. For the strongly-convex-strongly-concave case, the lower bound has been proven by~\cite{ibrahim2019linear} and~\cite{zhang2019lower}. Some authors have studied the special case where the interaction between $\x$ and $\y$ is bilinear~\cite{chambolle2011first,chen2014optimal,du2019linear} and variance reduction algorithms for finite sum objectives~\cite{carmon2019variance,palaniappan2016stochastic}.

The special case where $f$ is quadratic has also been studied extensively in the numerical analysis community~\cite{bai2003hermitian,benzi2005numerical,bai2009optimal}. One of the most notable algorithms for quadratic saddle point problems is Hermitian-skew-Hermitian Split (HSS)~\cite{bai2003hermitian}. However, most existing work do not provide a bound on the overall number of matrix-vector products.

The convex-concave saddle point problem can also be seen as a special case of variational inequalities with Lipschitz monotone operators~\cite{nemirovski2004prox,kinderlehrer1980introduction,gidel2018a,hsieh2019convergence,tseng1995linear}. Some existing algorithms for the saddle point problem, such as ExtraGradient, achieve the optimal rate in this more general setting as well~\cite{nemirovski2004prox,tseng1995linear}.




Going beyond the convex-concave setting, some researchers have also studied the nonconvex-concave case recently ~\cite{lin2019gradient,thekumparampil2019efficient,rafique2018non,lin2020near,lu2019hybrid,nouiehed2019solving,ostrovskii2020efficient}, with the goal being finding a stationary point of the nonconvex function $\phi(\x):=\max_\y f(\x,\y)$. By reducing to the strongly convex-strongly concave setting,~\cite{lin2020near} has achieved state-of-the-art results for nonconvex-concave problems.


\section{Linear Convergence and Refined Dependence on $L_{\x\y}$ in General Cases}

\subsection{Alternating Best Response}
Let us first consider the extreme case where $L_{\x\y}=0$. In this case, there is no interaction between $\x$ and $\y$, and $f(\x,\y)$ can be simply written as $h_1(\x)-h_2(\y)$, where $h_1$ and $h_2$ are strongly convex functions. Thus, in this case, the following trivial algorithm solves the problem
$$\x^*\gets \argmin_\x f(\x,\y_0), \quad \y^*\gets \argmax_\y f(\x^*,\y).$$
In other words, the equilibrium can be found by directly playing the best response to each other once. 

Now, let us consider the case where $L_{\x\y}$ is nonzero but small. In this case, would the best response dynamics converge to the saddle point? Specifically, consider the following procedure:
\begin{equation}
\label{equ:br}
\begin{cases}
\x_{t+1}&\gets \argmin_\x \{f(\x,\y_t)\}\\
\y_{t+1}&\gets \argmax_\y \{f(\x_{t+1},\y)\}
\end{cases}.
\end{equation}
Let us define $\y^*(\x):=\argmax_\y f(\x,\y)$ and $\x^*(\y):=\argmin_\x f(\x,\y)$. Because $\y^*(\x)$ is $L_{\x\y}/m_\y$-Lipschitz and $\x^*(\y)$ is $L_{\x\y}/m_\x$-Lipschitz~\footnote{See Fact 1 in Appendix~\ref{append:facts} for proof.},
\begin{align*}
\Vert \x_{t+1}-\x^*\Vert &= \Vert\x^*(\y_t)-\x^*(\y^*)\Vert \le \frac{L_{\x\y}}{m_\x}\Vert \y_t-\y^*\Vert\\
&=\frac{L_{\x\y}}{m_\x}\Vert \y^*(\x_t)-\y^*(\x^*)\Vert \le \frac{L_{\x\y}^2}{m_\x m_\y}\Vert\x_t-\x^*\Vert.
\end{align*}
Thus, when $L_{\x\y}^2<m_\x m_\y$, (\ref{equ:br}) is indeed a contraction. In fact, we can further replace the exact solution of the inner optimization problems with Nesterov's Accelerated Gradient Descent (AGD) for constant number of steps, as described in Algorithm~\ref{algo:br}.

\begin{algorithm}[t]
	\caption{Alternating Best Response (ABR)}
	\label{algo:br}
	\begin{algorithmic}
		\Require $g(\cdot,\cdot)$, Initial point $\z_0=[\x_0;\y_0]$, precision $\epsilon$, parameters $m_\x$, $m_\y$, $L_\x$, $L_\y$
		\State $\kappa_\x:=L_\x/m_\x$, $\kappa_\y:=L_\y/m_\y$, $T\gets \left\lceil \log_2\left(\frac{4\sqrt{\kappa_\x+\kappa_\y}}{\epsilon}\right)\right\rceil$ 
		\For{$t=0,\cdots,T$}
		    \State Run AGD on $g(\cdot,\y_t)$ from $\x_t $ for $\Theta(\sqrt{\kappa_\x}\ln(\kappa_\x))$ steps to get $\x_{t+1}$
		    \State Run AGD on $-g(\x_{t+1},\cdot)$ from $\y_t$ for $\Theta(\sqrt{\kappa_\y}\ln(\kappa_\y))$ steps to get $\y_{t+1}$
		\EndFor
	\end{algorithmic}
\end{algorithm}

The following theorem holds for the Alternating Best Response algorithm. The proof of the theorem, as well as a detailed version of Algorithm~\ref{algo:br} can be found in Appendix~\ref{append:brproof}.
\begin{theorem}
	\label{thm:br}
	If $g\in\mathcal{F}(m_\x,m_\y,L_\x,L_{\x\y},L_\y)$ and $L_{\x\y}\le\frac{1}{2}\sqrt{m_\x m_\y}$, Alternating Best Response returns $(\x_T,\y_T)$ such that 
	$$\Vert \x_T-\x^*\Vert+\Vert \y_T-\y^*\Vert\le \epsilon\left(\Vert \x_0-\x^*\Vert+\Vert \y_0-\y^*\Vert\right),$$
    and the number of gradient evaluations is bounded by
	(with $\kappa_\x=L_\x/m_\x$, $\kappa_\y=L_\y/m_\y$)
	$$O\left(\left(\sqrt{\kappa_\x+\kappa_\y}\right)\cdot \ln\left(\kappa_\x\kappa_\y\right)
	\ln\left(\kappa_\x\kappa_\y/\epsilon\right)\right).$$
\end{theorem}
Note that when $L_{\x\y}$ is small, Zhang et al's lower bound~\cite{zhang2019lower} can be written as $\Omega\left(\sqrt{\kappa_\x+\kappa_\y}\ln(1/\epsilon)\right)$. Thus Alternating Best Response matches this lower bound up to logarithmic factors.


\subsection{Accelerated Proximal Point for Minimax Optimization}
In the previous subsection, we showed that Alternating Best Response matches the lower bound when the interaction term $L_{\x\y}$ is sufficiently small. However, in order to apply the algorithm to functions with $L_{\x\y}>\frac{1}{2}\sqrt{m_\x m_\y}$, we need another algorithmic component, namely the accelerated proximal point algorithm~\cite{guler1992new,lin2020near}.


For a minimax optimization problem $\min_\x\max_\y f(\x,\y)$, define $\phi(\x):=\max_\y f(\x,\y)$. Suppose that we run the accelerated proximal point algorithm on $\phi(\x)$ with proximal parameter $\beta$: then the number of iterations can be easily bounded, while in each iteration one needs to solve a proximal problem $\min_\x \left\{\phi(\x)+\beta\Vert\x-\hat\x_t\Vert^2\right\}$. The key observation is that, this is equivalent to solving a minimax optimization problem $\min_\x \max_\y\left\{f(\x,\y)+\beta\Vert\x-\hat\x_t\Vert^2\right\}$. Thus, via accelerated proximal point, we are able to reduce solving $\min_\x \max_\y f(\x,\y)$ to solving $\min_\x \max_\y\left\{f(\x,\y)+\beta\Vert\x-\hat\x_t\Vert^2\right\}$.

This is exactly the idea behind Algorithm~\ref{algo:minmaxappa} (the idea was also used in \cite{lin2020near}). In the algorithm, $M$ is a positive constant characterizing the precision of solving the subproblem, where we require $M\ge \mathrm{poly}(\frac{L}{m_\x},\frac{L}{m_\y},\frac{\beta}{m_\x})$. If $M\to\infty$, the algorithm exactly becomes an instance of accelerated proximal point on $\phi(\x)=\max_\y f(\x,\y)$.


\begin{algorithm}
	\caption{Accelerated Proximal Point Algorithm for Minimax Optimization}
	\label{algo:minmaxappa}
	\begin{algorithmic}
		\Require Initial point $\z_0=[\x_0;\y_0]$, proximal parameter $\beta$, strongly-convex modulus $m_\x$
		\State $\hat\x_0\gets \x_0$, $\kappa\gets \beta/m_\x$, $\theta\gets \frac{2\sqrt{\kappa}-1}{2\sqrt{\kappa}+1}$, $\tau\gets \frac{1}{2\sqrt{\kappa}+4\kappa}$
		\For{$t=1,\cdots,T$}
		\State Suppose $(\x^*_t,\y^*_t)=\min_\x\max_\y f(\x,\y)+\beta\Vert \x-\hat\x_{t-1}\Vert^2$. 
		Find $(\x_t,\y_t)$ such that
		$$\Vert \x_t-\x^*_t\Vert+\Vert\y_t-\y^*_t\Vert\le \frac{1}{M}\left(\Vert \x_{t-1}-\x^*_t\Vert+\Vert\y_{t-1}-\y^*_t\Vert\right) $$
		\State $\hat\x_t\gets \x_t+\theta(\x_t-\x_{t-1})+\tau(\x_t-\hat\x_{t-1})$
		\EndFor
	\end{algorithmic}
\end{algorithm}

The following theorem can be shown for Algorithm~\ref{algo:minmaxappa}.
The proof can be found in Appendix~\ref{append:thm2proof}, and is based on the proof of Theorem 4.1 in~\cite{lin2020near}.
\begin{theorem}
\label{thm:appa}
The number of iterations needed by Algorithm~\ref{algo:minmaxappa} to produce $(\x_T,\y_T)$ such that $$\Vert \x_T-\x^*\Vert+\Vert \y_T-\y^*\Vert \le \epsilon\left(\Vert \x_0-\x^*\Vert+\Vert \y_0-\y^*\Vert \right)$$
is at most ($\kappa=\beta/m_\x$)
\begin{equation}
\hat T=8\sqrt{\kappa}\cdot\ln\left(\frac{28\kappa^2 L}{m_\y}\sqrt{\frac{L^2}{m_\x m_\y}}\cdot\frac{1}{\epsilon}\right).
\end{equation}
\end{theorem}

\subsection{Proximal Alternating Best Response}
With the two algorithmic components, namely Alternating Best Response and Accelerated Proximal Point in place, we can now combine them and design an efficient algorithm for general strongly convex-strongly concave functions. The high-level idea is to exploit the accelerated proximal point algorithm twice to reduce a general problem into one solvable by Alternating Best Response.

To start with, let us consider a strongly-convex-strongly-concave function $f(\x,\y)$, and apply Algorithm~\ref{algo:minmaxappa} for $f$ with proximal parameter $\beta=L_{\x\y}$. By Theorem~\ref{thm:appa}, 
the algorithm can converge in $\tilde{O}\left(\sqrt{\frac{L_{\x\y}}{m_\x}}\right)$ iterations, while in each iteration we need to solve a regularized minimax problem
$$\min_\x \max_\y \left\{f(\x,\y)+\beta\Vert\x-\hat\x_{t-1}\Vert^2\right\}.$$
This is equivalent to $\min_\y \max_\x \left\{-f(\x,\y)-\beta\Vert\x-\hat\x_{t-1}\Vert^2\right\}$~\footnote{Although Sion's Theorem does not apply here as we considered unconstrained problem, we can still exchange the order since the function is strongly-convex-strongly-concave~\cite{hartung1982extension}.}, so we can apply Algorithm~\ref{algo:minmaxappa} once more to this problem with parameter $\beta=L_{\x\y}$. This procedure would require $\tilde{O}\left(\sqrt{\frac{L_{\x\y}}{m_\y}}\right)$ iterations, and in each iteration, one need to solve a minimax problem of the form
\begin{align*}
  &\min_\y \max_\x \left\{-f(\x,\y)-\beta\Vert\x-\hat\x_{t-1}\Vert^2+\beta\Vert\y-\hat\y_{t'-1}\Vert^2\right\}\\
  =&-\min_\x\max_\y\left\{f(\x,\y)+\beta\Vert\x-\hat\x_{t-1}\Vert^2-\beta\Vert\y-\hat\y_{t'-1}\Vert^2\right\}.
\end{align*}
Hence, we reduced the original problem to a problem that is $2\beta$-strongly convex with respect to $\x$ and $2\beta$-strongly concave with respect to $\y$. 
Now the interaction between $\x$ and $\y$ is (relatively) much weaker and one can easily see that $L_{\x\y}\le \frac{1}{2}\sqrt{2\beta\cdot 2\beta}$.
Consequently the final problem can be solved in $\tilde{O}\left(\frac{L_\x}{L_{\x\y}}\right)$ gradient evaluations using the Alternating Best Response algorithm. 
We first consider the case where $L_{\x\y}>\max\{m_\x,m_\y\}$.
The total gradient complexity would thus be
\begin{align*}
\tilde{O}\left(\sqrt{\frac{L_{\x\y}}{m_\x}}\right)\cdot \tilde{O}\left(\sqrt{\frac{L_{\x\y}}{m_\y}}\right)\cdot \tilde{O}\left(\sqrt{\frac{L}{L_{\x\y}}}\right)=\tilde{O}\left(\sqrt{\frac{L\cdot L_{\x\y}}{m_\x m_\y}}\right).
\end{align*}
In order to deal with the case where $L_{\x\y}<\max\{m_\x,m_\y\}$, we shall choose $\beta_1=\max\{L_{\x\y},m_\x\}$ for the first level of proximal point, and $\beta_2=\max\{L_{\x\y},m_\y\}$ for the second level of proximal point. In this case, the total gradient complexity bound can be shown to be
\begin{align*}
\tilde{O}\left(\sqrt{\frac{\beta_1}{m_\x}}\right)\cdot \tilde{O}\left(\sqrt{\frac{\beta_2}{m_\y}}\right)\cdot \tilde{O}\left(\sqrt{\frac{L}{\beta_1}+\frac{L}{\beta_2}}\right)=\tilde{O}\left(\sqrt{\frac{L_\x}{m_\x}+\frac{L\cdot L_{\x\y}}{m_\x m_\y}+\frac{L_\y}{m_\y}}\right).
\end{align*}
A formal description of the algorithm is provided in Algorithm~\ref{algo:pbr}, and a formal statement of the complexity upper bound is provided in Theorem~\ref{thm:phr}.
The proof is deferred to Appendix~\ref{append:thm3proof}.
\begin{theorem}
\label{thm:phr}
Assume that $f\in\mathcal{F}(m_\x,m_\y,L_\x,L_{\x\y},L_\y)$. In Algorithm~\ref{algo:pbr}, the gradient complexity to produce $(\x_T,\y_T)$ such that $\Vert \z_T-\z^*\Vert\le \epsilon$ is
\begin{equation*}
O\left(\sqrt{\frac{L_\x}{m_\x}+\frac{L\cdot L_{\x\y}}{m_\x m_\y}+\frac{L_\y}{m_\y}}\cdot\ln^3\left(\frac{L^2}{m_\x m_\y}\right)\ln\left(\frac{L^2}{m_\x m_\y}\cdot\frac{\Vert\z_0-\z^*\Vert}{\epsilon}\right)\right).
\end{equation*}
\end{theorem}

\begin{algorithm}
	\caption{APPA-ABR}
	\label{algo:pbr2}
	\begin{algorithmic}[1]
		\Require $g(\cdot,\cdot)$, Initial point $\z_0=[\x_0;\y_0]$, precision parameter $M_1$
		
		\State $\beta_2\gets \max\{m_\y, L_{\x\y}\}$, $M_2\gets \frac{96L^{2.5}}{m_\x m_\y^{1.5}}$
		\State $\hat\y_0\gets \y_0$, $\kappa\gets \beta_2/m_\y$, $\theta\gets \frac{2\sqrt{\kappa}-1}{2\sqrt{\kappa}+1}$, $\tau\gets \frac{1}{2\sqrt{\kappa}+4\kappa}$, $t\gets 0$
		\Repeat
		\State $t\gets t+1$
		\State $(\x_t,\y_t)\gets$ABR($g(\x,\y)-\beta_2\Vert\y-\hat\y_{t-1}\Vert^2,[\x_{t-1};\y_{t-1}]$, $1/M_2$, $2\beta_1$, $2\beta_2$, $3L$, $3L$)
		\State $\hat\y_t\gets \y_t+\theta(\y_t-\y_{t-1})+\tau(\y_t-\hat\y_{t-1})$
		\Until{$\Vert \nabla g(\x_t,\y_t)\Vert\le \frac{\min\{m_\x,m_\y\}}{9LM_1}\Vert \nabla g(\x_0,\y_0)\Vert$}
	\end{algorithmic}
\end{algorithm}

\begin{algorithm}
	\caption{Proximal Best Response}
	\label{algo:pbr}
	\begin{algorithmic}[1]
		\Require Initial point $\z_0=[\x_0;\y_0]$
		\State $\beta_1\gets \max\{m_\x, L_{\x\y}\}$, $M_1\gets \frac{80L^3}{m_\x^{1.5} m_\y^{1.5}}$
		\State $\hat\x_0\gets \x_0$, $\kappa\gets \beta_1/m_\x$, $\theta\gets \frac{2\sqrt{\kappa}-1}{2\sqrt{\kappa}+1}$, $\tau\gets \frac{1}{2\sqrt{\kappa}+4\kappa}$
		\For{$t=1,\cdots,T$}
		\State $(\x_t,\y_t)\gets$ APPA-ABR($f(\x,\y)+\beta_1\Vert\x-\hat\x_{t-1}\Vert^2$, $[\x_{t-1},\y_{t-1}]$, $M_1$)
		\State $\hat\x_t\gets \x_t+\theta(\x_t-\x_{t-1})+\tau(\x_t-\hat\x_{t-1})$
		\EndFor
	\end{algorithmic}
\end{algorithm}

\subsection{Implications of Theorem~\ref{thm:phr}}

Theorem~\ref{thm:phr} improves over the results of Lin et al. in two ways. First, Lin et al.'s upper bound has a $\ln^3(1/\epsilon)$ factor, while our algorithm enjoys linear convergence. Second, our result has a better dependence on $L_{\x\y}$. To see this, note that when $L_{\x\y}\ll L$, $\frac{L_\x}{m_\x}+\frac{L\cdot L_{\x\y}}{m_\x m_\y}+\frac{L_\y}{m_\y}\ll \frac{L_\x}{m_\x}+\frac{L^2}{m_\x m_\y}+\frac{L_\y}{m_\y} \le \frac{3L^2}{m_\x m_\y}.$ This is also illustrated by Fig.~\ref{fig:comparison}, where Proximal Best Response (the red line) significantly outperforms Lin et al.'s result (the blue line) when $L_{\x\y}\ll L$. In particular, Proximal Best Response matches the lower bound when $L_{\x\y}>L_\x$ or when $L_{\x\y}<\max\{m_\x,m_\y\}$; in between, it is able to gracefully interpolate the two cases.

As shown by Lin et al.~\cite{lin2020near}, convex-concave problems and strongly convex-concave problems can be reduced to strongly convex-strongly concave problems. Hence, Theorem~\ref{thm:phr} naturally implies improved algorithms for convex-concave and strongly convex-concave problems.

\begin{corollary}
\label{cor:scc}
If $f(\x,\y)$ is $(L_\x,L_{\x\y},L_\y)$-smooth and $m_\x$-strongly convex w.r.t. $\x$, via reduction to Theorem~\ref{thm:phr}, the gradient complexity of finding an $\epsilon$-saddle point is $\tilde{O}\Bigl(\sqrt{\frac{m_\x\cdot L_\y+L\cdot L_{\x\y}}{m_\x \epsilon}}\Bigr)$.
\end{corollary}

\begin{corollary}
\label{cor:scc2}
If $f(\x,\y)$ is $(L_\x,L_{\x\y},L_\y)$-smooth and convex-concave, via reduction to Theorem~\ref{thm:phr}, the gradient complexity to produce an $\epsilon$-saddle point is $\tilde{O}\Bigl(\sqrt{\frac{L_\x+L_\y}{\epsilon}}+\frac{\sqrt{L\cdot L_{\x\y}}}{\epsilon}\Bigr)$.
\end{corollary}
The precise statement as well as the proofs can be found in Appendix~\ref{append:implication}. We remark that the reduction is for constrained minimax optimization, and Theorem~\ref{thm:phr} holds for constrained problems after simple modifications to the algorithm.

\section{Near Optimal Dependence on $L_{\x\y}$ in Quadratic Cases}
\label{sec:rhss}
We can see that proximal best response has near optimal dependence on condition numbers when $L_{\x\y}>L_\x$ or when $L_{\x\y}<\max\{m_\x,m_\y\}$. However, when $L_{\x\y}$ falls in between, there is still a significant gap between the upper bound and the lower bound.
In this section, we try to close this gap for quadratic functions; i.e. we assume that
\begin{equation}
\label{equ:quadratic}
f(\x,\y)=\frac{1}{2}\x^T\mathbf{A}\x +\x^T\mathbf{B}\y-\frac{1}{2}\y^T\mathbf{C}\y + \mathbf{u}^T\x + \mathbf{v}^T\y.
\end{equation} 
The reason to consider quadratic functions is threefold. First, the lower bound instance by~\cite{zhang2019lower} is a quadratic function; thus, this lower bound applies to quadratic functions as well, so it would be interesting to match the lower bound for quadratic functions first. Second, quadratic functions are considerably easier to analyze. Third, finding the saddle point of quadratic functions is an important problem on its own, and has many applications (see~\cite{benzi2005numerical} and references therein).

Our assumption that $f\in \mathcal{F}(m_\x,m_\y,L_\x,L_{\x\y},L_\y)$ now becomes assumptions on the singular values of matrices: $m_\x \iden\preccurlyeq \ma\preccurlyeq L_\x \iden$, $m_\y\iden\preccurlyeq \mc\preccurlyeq L_\y \iden$, $\Vert \mb\Vert_2\le L_{\x\y}$. In this case, the unique saddle point is given by the solution to a linear system
\begin{equation*}
\left[\begin{matrix}
\x^* \\ \y^*
\end{matrix}\right]=\jacobian^{-1}\bb=\left[\begin{matrix}
\ma & \mb \\ -\mb^T & \mc
\end{matrix}\right]^{-1}
\left[\begin{matrix}
-\mathbf{u} \\ \mathbf{v}
\end{matrix}\right].
\end{equation*}

Throughout this section we assume that $L_\x=L_\y$ and $m_\x<m_\y$, which are without loss of generality, and that $m_\y<L_{\x\y}$, as otherwise proximal best response is already near-optimal.

\subsection{Hermitian-Skew-Hermitian-Split}

We now focus on how to solve the linear system $\jacobian\z=\bb$, where  $\jacobian:=\left[\begin{matrix}
\ma & \mb \\ -\mb^T & \mc
\end{matrix}\right]$ is positive definite but not symmetric. A straightforward way to solve this asymmetric linear system is apply conjugate gradient to solve the normal equation $\jacobian^T\jacobian\z = \jacobian^T\bb$. However the complexity of this approach is $O\left(\frac{L}{\min\{m_\x,m_\y\}}\right)$, which is much worse than the lower bound. Instead, we utilize the Hermitian-Skew-Hermitian Split (HSS) algorithm~\cite{bai2003hermitian}, which is designed to solve positive definite asymmetric systems. Define
$$\mathbf{G}:=\left[\begin{matrix}
\ma & 0 \\ 0 & \mc
\end{matrix}\right],\;
\mathbf{S}:=\left[\begin{matrix}
0 & \mb \\ -\mb^T & 0
\end{matrix}\right],\;\mathbf{P}:=\left[\begin{matrix}
	\alpha\iden+\beta\ma &  \\  & \iden+\beta\mc
\end{matrix}\right],$$
where $\alpha$ and $\beta$ are constants to be determined. Let $\z_t:=[\x_t;\y_t]$. Then HSS runs as
\begin{equation}
\label{equ:rule}
\begin{cases}
(\eta \matp + \mg) \z_{t+1/2} &= (\eta \matp - \mathbf{S})\z_{t} + \mathbf{b},\\
(\eta \matp + \mathbf{S}) \z_{t+1} &= (\eta \matp - \mathbf{G})\z_{t+1/2} + \mathbf{b}.
\end{cases}
\end{equation}
Here $\eta>0$ is another constant. In this procedure, it can be shown that
\begin{equation*}
\z_{t+1}-\z^*=\left(\eta\matp+\mathbf{S}\right)^{-1}\left(\eta\matp-\mathbf{G}\right)\left(\eta\matp+\mathbf{G}\right)^{-1}\left(\eta\matp-\mathbf{S}\right)(\z_t-\z^*).
\end{equation*}
The key observation of HSS is that the equation above is a contraction.
\begin{lemma}[\cite{bai2003hermitian}]
	\label{lem:hsscontraction}
Define $M(\eta):=\left(\eta\matp+\mathbf{S}\right)^{-1}\left(\eta\matp-\mathbf{G}\right)\left(\eta\matp+\mathbf{G}\right)^{-1}\left(\eta\matp-\mathbf{S}\right)$. Then\footnote{Here $\rho(\cdot)$ stands for the spectral radius of a matrix, and $sp(\cdot)$ stands for its spectrum.}
	\begin{equation*}
	\rho(\mathbf{M}(\eta))\le \Vert \mathbf{M}(\eta)\Vert_2\le \max_{\lambda_i\in sp(\matp^{-1}\mathbf{G})}\left|\frac{\lambda_i-\eta}{\lambda_i+\eta}\right|<1.
	\end{equation*}
\end{lemma}
Lemma~\ref{lem:hsscontraction} provides an upper bound on the iteration complexity of HSS, as in the original analysis of HSS~\cite{bai2003hermitian}. However, it does not consider the computational cost per iteration. In particular, the matrix $\eta\matp+\mathbf{S}$ is also asymmetric, and in fact corresponds to another quadratic minimax optimization problem. The original HSS paper did not consider how to solve this subproblem for general $\matp$. Our idea is to solve the subproblem recursively, as explained in the next subsection.

\subsection{Recursive HSS}

\begin{algorithm}[htb]
	\caption{RHSS($k$) (Recursive Hermitian-skew-Hermitian Split)}
	\begin{algorithmic}
		\Require Initial point $[\x_0;\y_0]$, precision $\epsilon$, parameters $m_\x$, $m_\y$, $L_{\x\y}$
		\State $t\gets 0$, $M_1\gets \frac{192L^5}{m_\x^2 m_\y^3}$, $M_2\gets \frac{16L_{\x\y}}{m_\y}$, $\alpha\gets \frac{m_\x}{m_\y}$, $\beta\gets L_{\x\y}^{-\frac{2}{k}}m_\y^{-\frac{k-2}{k}}$, $\eta\gets L_{\x\y}^{\frac{1}{k}}m_\y^{1-\frac{1}{k}}$, $\tilde{\epsilon}\gets \frac{m_\x\epsilon}{L_{\x\y}+L_\x}$
		\Repeat
		$$\left[\begin{matrix}
		\mathbf{r}_1\\
		\mathbf{r}_2
		\end{matrix}\right]\gets \left[\begin{matrix}
		\eta\left(\alpha\iden+\beta\ma\right) & -\mb \\ \mb^T & \eta\left(\iden+\beta\mc\right)
		\end{matrix}\right]\left[\begin{matrix}
		\mathbf{x}_t\\
		\mathbf{y}_t
		\end{matrix}\right]+\left[\begin{matrix}
		-\mathbf{u}\\
		\mathbf{v}
		\end{matrix}\right]$$
		\State Use conjugate gradient with initial point $[\x_{t};\y_t]$ and precision $1/M_1$ to solve $$\left[\begin{matrix}
		\x_{t+1/2}\\
		\y_{t+1/2}
		\end{matrix}\right]\gets\left[\begin{matrix}
		\eta\left(\alpha\iden+\beta\ma\right)+\ma & \\  & \eta\left(\iden+\beta\mc\right)+\mc
		\end{matrix}\right]^{-1}\left[\begin{matrix}
		\mathbf{r}_1\\
		\mathbf{r}_2
		\end{matrix}\right]$$
		
		$$\left[\begin{matrix}
		\mathbf{{w}}_1\\
		\mathbf{{w}}_2
		\end{matrix}\right]\gets \left[\begin{matrix}
		\eta\alpha\iden+\eta\beta\ma-\ma & 0 \\ 0 & \eta\left(\iden+\beta\mc\right)-\mc
		\end{matrix}\right]\left[\begin{matrix}
		{\x}_{t+1/2}\\
		{\y}_{t+1/2}
		\end{matrix}\right]+\left[\begin{matrix}
		-\mathbf{u}\\
		\mathbf{v}
		\end{matrix}\right]$$
		\State Call RHSS($k-1$) with initial point $[\x_t;\y_t]$ and precision $1/M_2$ to solve
		$$\left[\begin{matrix}
		\x_{t+1}\\
		\y_{t+1}
		\end{matrix}\right]\gets \left[\begin{matrix}
		\eta\left(\alpha\iden+\beta\ma\right) & \mb\\  -\mb^T& \eta\left(\iden+\beta\mc\right)
		\end{matrix}\right]^{-1}\left[\begin{matrix}
		\mathbf{w}_1\\
		\mathbf{w}_2
		\end{matrix}\right]$$
		\State $t\gets t+1$
		\Until{$\Vert \jacobian\z_{t}-\bb\Vert\le \tilde{\epsilon}\Vert\jacobian\z_0-\bb\Vert$}
	\end{algorithmic}
\label{algo:rhss}
\end{algorithm}

In this subsection, we describe our algorithm Recursive Hermitian-skew-Hermitian Split, or RHSS($k$), which uses HSS in $k-1$ levels of recursion. Specifically, RHSS($k$) calls HSS with parameters $\alpha=m_\x/m_\y$, $\beta=L_{\x\y}^{-\frac{2}{k}}m_\y^{-\frac{k-2}{k}}$, $\eta=L_{\x\y}^{\frac{1}{k}}m_\y^{\frac{k-1}{k}}$. In each iteration, it solves two linear systems. The first one, which is associated with $\eta\matp+\mathbf{G}$, can be solved with Conjugate Gradient~\cite{hestenes1952methods} as $\eta\matp+\mathbf{G}$ is symmetric positive definite. The second one is associated with
$$\eta\matp+\mathbf{S}=\left[\begin{matrix}
\eta\left(\alpha\iden+\beta\ma\right) & \mb \\ -\mb^T & \eta\left(\iden+\beta\mc\right)
\end{matrix}\right],$$
which is equivalent to a quadratic minimax optimization problem. RHSS($k$) then makes a recursive call RHSS($k-1$) to solve this subproblem. When $k=1$, we  simply run the Proximal Best Response algorithm (Algorithm~\ref{algo:pbr}). A detailed description of RHSS($k$) for $k\ge 2$ is given in Algorithm~\ref{algo:rhss}.

Our main result for RHSS($k$) is the following theorem. Note that for an algorithm on quadratic functions, the number of matrix-vector products is the same as the gradient complexity.
\begin{theorem}
\label{thm:rhss}
There exists constants $C_1$, $C_2$, such that the number of matrix-vector products needed to find $(\x_T,\y_T)$ such that $\Vert\z_T-\z^*\Vert\le \epsilon$ is at most
\begin{equation}
\label{equ:rhssbound}
\sqrt{\frac{L_{\x\y}^2}{m_\x m_\y}+\left(\frac{L_\x}{m_\x}+\frac{L_\y}{m_\y}\right)\left(1+\left(\frac{L_{\x\y}}{\max\{m_\x,m_\y\}}\right)^{\frac{1}{k}}\right)}\cdot \left(C_1\ln\left(\frac{C_2 L^2}{m_\x m_\y}\right)\right)^{k+3}\ln\left(\frac{\Vert \z_0-\z^*\Vert}{\epsilon}\right).
\end{equation}
\end{theorem}
If $k$ is chosen as a fixed constant, the comparison of (\ref{equ:rhssbound}) and the lower bound \cite{zhang2019lower} is illustrated in Fig.~\ref{fig:comparison}. One can see that as $k$ increases, the upper bound of RHSS($k$) gradually fits the lower bound (as long as $k$ is a constant).

By optimizing $k$, we can also show the following corollary.
\begin{corollary}
\label{cor:rhss}
When $k=\Theta\left(\sqrt{\ln\left(\frac{L^2}{m_\x m_\y}\right)/\ln\ln\left(\frac{L^2}{m_\x m_\y}\right)}\right)$, the number of matrix vector products that RHSS($k$) needs to find $\z_T$ such that $\Vert\z_T-\z^*\Vert\le\epsilon$ is
\begin{equation*}
\sqrt{\frac{L_{\x\y}^2}{m_\x m_\y}+\frac{L_\x}{m_\x}+\frac{L_\y}{m_\y}}\cdot\ln\left(\frac{\Vert \z_0-\z^*\Vert}{\epsilon}\right)\cdot\left(\frac{L^2}{m_\x m_\y}\right)^{o(1)}.
\end{equation*}
\end{corollary}
In other words, for the quadratic saddle point problem, RHSS($k$) 
with the optimal choice of $k$
matches the lower bound up to a sub-polynomial factor.

The proof of both Theorem~\ref{thm:rhss} and Corollary~\ref{cor:rhss} can be found in Appendix~\ref{append:thm4proof}.

\section{Conclusion}
In this work, we studied convex-concave minimax optimization problems. For general strongly convex-strongly concave problems, our Proximal Best Response algorithm achieves linear convergence and better dependence on $L_{\x\y}$, the interaction parameter. Via known reductions~\cite{lin2020near}, this result implies better upper bounds for strongly convex-concave and convex-concave problems. For quadratic functions, our algorithm RHSS($k$) is able to match the lower bound up to a sub-polynomial factor.

In future research, one interesting direction is to extend RHSS($k$) to general strongly convex-strongly concave functions. Another important direction would be to shave the remaining sub-polynomial factor from the upper bound for quadratic functions.

\section*{Broader Impact}
This work is purely theoretical and does not present foreseeable societal consequences.

\begin{ack}
The research is supported in part by the National Natural Science Foundation of China Grant 61822203, 61772297, 61632016, 61761146003, and the Zhongguancun Haihua Institute for Frontier Information Technology, Turing AI Institute of Nanjing and Xi'an Institute for Interdisciplinary Information Core Technology. The authors thank Kefan Dong, Guodong Zhang and Chi Jin for helpful discussions. 
\end{ack}

\bibliographystyle{plain}
\bibliography{minmax.bib}

\appendix
\newpage

\section{Some Useful Properties}
\label{append:facts}
In this section, we review some useful properties of functions in $\mathcal{F}(m_\x,m_\y,L_\x,L_{\x\y},L_\y)$.
Some of the facts are known (see e.g., \cite{lin2020near}, \cite{zhang2019lower}) and we provide the proofs for completeness.

\begin{fact}
\label{prop:facts}
Suppose $f\in\mathcal{F}(m_\x,m_\y,L_\x,L_{\x\y},L_\y)$.
Let us define $\y^*(\x):=\argmax_\y f(\x,\y)$, $\x^*(\y):=\argmin_\x f(\x,\y)$, $\phi(\x):=\max_\y f(\x,\y)$ and $\psi(\y):=\min_\x f(\x,\y)$. Then, we have that
\begin{enumerate}
	\item $\y^*$ is $L_{\x\y}/m_\y$-Lipschitz, $\x^*$ is $L_{\x\y}/m_\x$-Lipschitz;
	\item $\phi(\x)$ is $m_\x$-strongly convex and $L_\x+L_{\x\y}^2/m_\y$-smooth; $\psi(\y)$ is $m_\y$-strongly concave and $L_\y+L_{\x\y}^2/m_\x$-smooth.
\end{enumerate}
\end{fact}
\begin{proof}
1. Consider arbitrary $\x$ and $\x'$. By definition, $\nabla_\y f(\x,\y^*(\x))=\nabla_\y f(\x',\y^*(\x'))=\zero$.  By the definition of $(L_\x,L_{\x\y},L_\y)$-smoothness, $\Vert \nabla_\y f(\x',\y^*(\x))\Vert \le L_{\x\y}\Vert \x-\x'\Vert$. Thus
\begin{equation*}
    m_\y \Vert\y^*(\x)-\y^*(\x')\Vert \le \Vert \nabla_\y f(\x',\y^*(\x))\Vert \le L_{\x\y}\Vert \x-\x'\Vert.
\end{equation*}
This proves that $y^*(\cdot)$ is $L_{\x\y}/m_\y$-Lipschitz. Similarly $\x^*(\cdot)$ is $L_{\x\y}/m_\x$-Lipschitz.

2. By Danskin's Theorem, $\nabla \phi(\x)=\nabla_\x f(\x,\y^*(\x))$. Thus, $\forall \x,\x'$
\begin{align*}
        \Vert \nabla\phi(\x)-\nabla\phi(\x')\Vert &= \Vert \nabla_\x f(\x,\y^*(\x))-\nabla_\x f(\x',\y^*(\x'))\Vert\\
    &\le \Vert \nabla_\x f(\x,\y^*(\x))-\nabla_\x f(\x,\y^*(\x'))\Vert+\Vert\nabla_\x f(\x,\y^*(\x'))-\nabla_\x f(\x',\y^*(\x'))\Vert\\
    &\le L_{\x\y}\cdot \Vert \y^*(\x)-\y^*(\x')\Vert + L_\x \Vert \x-\x'\Vert\\
    &\le \left(L_\x+\frac{L_{\x\y}^2}{m_\y}\right)\Vert \x-\x'\Vert.
\end{align*}
On the other hand, $\forall \x,\x'$,
\begin{align*}
\phi(\x')-\phi(\x)-(\x'-\x)^T\nabla \phi(\x)&=f(\x',\y^*(\x'))-f(\x,\y^*(\x))-(\x'-\x)^T\nabla_\x f(\x,\y^*(\x))\\
&\ge f(\x',\y^*(\x))-f(\x,\y^*(\x))-(\x'-\x)^T\nabla_\x f(\x,\y^*(\x))\\
&\ge \frac{m_\x}{2}\Vert \x'-\x\Vert^2.
\end{align*}
Thus $\phi(\x)$ is $m_\x$-strongly convex and $\Bigl(L_\x+\frac{L_{\x\y}^2}{m_\y}\Bigr)$-smooth. By symmetric arguments, one can show that $\psi(\y)$ is $m_\y$-strongly concave and $\Bigl(L_\y+\frac{L_{\x\y}^2}{m_\x}\Bigr)$-smooth.
\end{proof}

\begin{fact}
\label{fact:z}
Let $\z:=[\x;\y]$ and $\z^*:=[\x^*;\y^*]$. Then
$$\frac{1}{\sqrt{2}}\left( \Vert \x-\x^*\Vert+\Vert\y-\y^*\Vert \right)\le \Vert \z-\z^*\Vert \le \Vert \x-\x^*\Vert+\Vert\y-\y^*\Vert.$$
\end{fact}
\begin{proof}
This can be easily proven using the AM-GM inequality.
\end{proof}

\begin{fact}
\label{fact:gradnorm}
Let $\z:=[\x;\y]\in\R^{m+n}$, $\z^*:=[\x^*;\y^*]$. Then
$$\min\{m_\x,m_\y\}\Vert\z-\z^*\Vert\le \Vert\nabla f(\x,\y)\Vert \le 2L\Vert \z-\z^*\Vert.$$
\end{fact}
\begin{proof}
By properties of strong convexity~\cite{nesterov2013introductory}, $\forall \x,\y$
$$f(\x,\y^*(\x))-f(\x,\y)\le \frac{1}{2m_\y}\Vert \nabla_\y f(\x,\y)\Vert^2.$$
Similarly,
$$f(\x,\y)-f(\x^*(\y),\y)\le \frac{1}{2m_\x}\Vert \nabla_\x f(\x,\y)\Vert^2.$$
Thus,
\begin{align*}
\Vert\nabla f(\x,\y)\Vert^2&=\Vert \nabla_\x f(\x,\y)\Vert^2+\Vert\nabla_\y f(\x,\y)\Vert^2\\
&\ge 2\min\{m_\x,m_\y\}\left(\phi(\x)-\psi(\y)\right).
\end{align*}
Here $\phi(\cdot)=\max_\y f(\cdot,\y)$, $\psi(\cdot)=\min_\x f(\x,\cdot)$. By Proposition~\ref{prop:facts}, $\phi$ is $m_\x$-strongly convex while $\psi$ is $m_\y$-strongly concave. Hence
\begin{equation*}
\phi(\x)-\psi(\y)\ge \frac{\min\{m_\x,m_\y\}}{2}\left(\Vert\x-\x^*\Vert^2+\Vert\y-\y^*\Vert^2\right)=\frac{\min\{m_\x,m_\y\}}{2}\Vert\z-\z^*\Vert^2.
\end{equation*}
It follows that $\Vert \nabla f(\x,\y)\Vert \ge \min\{m_\x,m_\y\}\Vert \z-\z^*\Vert$. On the other hand,
\begin{align*}
\Vert \nabla_\x f(\x,\y)\Vert &\le L_{\x\y}\Vert\y-\y^*\Vert + L_\x \Vert\x-\x^*\Vert,\\
\Vert \nabla_\y f(\x,\y)\Vert &\le L_{\x\y}\Vert\x-\x^*\Vert + L_\y\Vert \y-\y^*\Vert.
\end{align*}
As a result $\Vert \nabla f(\x,\y)\Vert^2\le L\left(\Vert\x-\x^*\Vert+\Vert\y-\y^*\Vert\right)^2\le 4L^2\Vert \z-\z^*\Vert^2$.

\end{proof}

\begin{fact}
\label{fact:dualitygap}
Let $\hat\z=[\hat\x;\hat\y]$. Then $\Vert\hat\z-\z^*\Vert\le \epsilon$ implies
$$\max_\y f(\hat\x,\y)-\min_\x f(\x,\hat\y)\le \frac{L^2}{\min\{m_\x,m_\y\}}\epsilon^2.$$
\end{fact}
\begin{proof}
Define $\phi(\x)=\max_\y f(\x,\y)$ and $\psi(\y)=\min_\x f(\x,\y)$. Then
\begin{align*}
\max_\y f(\hat\x,\y)-\min_\x f(\x,\hat\y) &= \phi(\hat\x)-\psi(\hat\y).
\end{align*}
By Fact~\ref{prop:facts}, $\phi$ is $(L_\x+L_{\x\y}^2/m_\x)$-smooth while $\psi$ is $(L_\y+L_{\x\y}^2/m_\x)$-smooth. Since $\phi(\x^*)=\psi(\y^*)$, $\nabla \phi(\x^*)=\zero$, $\nabla\psi(\y^*)=\zero$,
\begin{align*}
\phi(\hat\x)-\psi(\hat\y)&\le \frac{1}{2}\left(L_\x+\frac{L_{\x\y^2}}{m_\x}\right)\Vert\hat\x-\x^*\Vert^2+\frac{1}{2}\left(L_\y+\frac{L_{\x\y^2}}{m_\y}\right)\Vert\hat\y-\y^*\Vert^2\\
&\le \frac{1}{2}\left(L+\frac{L_{\x\y^2}}{\min\{m_\x,m_\y\}}\right)(\Vert \hat\x-\x^*\Vert^2+\Vert\hat\y-\y^*\Vert^2)\\
&\le \frac{L^2}{\min\{m_\x,m_\y\}}\epsilon^2.
\end{align*}
\end{proof}

\subsection{Accelerated Gradient Descent}
Nesterov's Accelerated Gradient Descent~\cite{nesterov1983a} is an optimal first-order algorithm for smooth and convex functions. Here we present a version of AGD for minimizing an $l$-smooth and $m$-strongly convex functions $g(\cdot)$. It is a crucial building block for the algorithms in this work.

\begin{algorithm}
\renewcommand{\thealgorithm}{I}
	\caption{AGD($g$, $\x_0$, $T$)~\cite{nesterov2013introductory}}
	\begin{algorithmic}
		\Require Initial point $\x_0$, smoothness constant $l$, strongly-convex modulus $m$, number of iterations $T$
		\State $\tilde\x_0\gets \x_0$, $\eta\gets 1/l$, $\kappa\gets l/m$, $\theta\gets (\sqrt{\kappa}-1)/(\sqrt{\kappa}+1)$
		\For{$t=1,\cdots,T$}
			\State $\x_t\gets \tilde{\x}_{t-1}-\eta\nabla g(\tilde{\x}_{t-1})$
			\State $\tilde{\x}_t\gets \x_t+\theta(\x_t-\x_{t-1})$
		\EndFor
	\end{algorithmic}
\end{algorithm}

The following classical theorem holds for AGD. It implies that the complexity is $O\left(\sqrt{\kappa}\ln\left(\frac{1}{\epsilon}\right)\right)$, which greatly improves over the $O\left(\kappa\ln\left(\frac{1}{\epsilon}\right)\right)$ bound for gradient descent.
\begin{lemma}
(\cite[Theorem 2.2.3]{nesterov2013introductory}) In the AGD algorithm,
\label{thm:nesterov}
\begin{equation*}
\Vert \x_T-\x^*\Vert^2\le (\kappa+1) \Vert \x_0-\x^*\Vert^2\cdot\left(1-\frac{1}{\sqrt{\kappa}}\right)^T.
\end{equation*}
\end{lemma}

\section{Proof of Theorem 1}
\label{append:brproof}
We will start by giving a precise statement of Algorithm~1.
\addtocounter{algorithm}{-6}
\begin{algorithm}[ht]
	\caption{Alternating Best Response (ABR)}
	\begin{algorithmic}
		\Require $g(\cdot,\cdot)$, Initial point $\z_0=[\x_0;\y_0]$, precision $\epsilon$, parameters $m_\x$, $m_\y$, $L_\x$, $L_\y$
		\State $\kappa_\x:=L_\x/m_\x$, $\kappa_\y:=L_\y/m_\y$, $T\gets \left\lceil \log_2\left(\frac{4\sqrt{\kappa_\x+\kappa_\y}}{\epsilon}\right)\right\rceil$
		\For{$t=0,\cdots,T$}
		\State $\x_{t+1}\gets $ AGD($g(\cdot,\y_t),\x_t,2\sqrt{\kappa_\x}\ln(24\kappa_\x))$
		\State $\y_{t+1}\gets $ AGD($-g(\x_{t+1},\cdot),\y_t,2\sqrt{\kappa_\y}\ln(24\kappa_\y))$
		\EndFor
	\end{algorithmic}
\end{algorithm}

We proceed to prove Theorem~\ref{thm:br}.

\begin{manualtheorem}{1}
	If $g\in\mathcal{F}(m_\x,m_\y,L_\x,L_{\x\y},L_\y)$ and $L_{\x\y}<\frac{1}{2}\sqrt{m_\x m_\y}$, Alternating Best Response returns $(\x_T,\y_T)$ such that 
	$$\Vert \x_T-\x^*\Vert+\Vert \y_T-\y^*\Vert\le \epsilon\left(\Vert \x_0-\x^*\Vert+\Vert \y_0-\y^*\Vert\right),$$
	using ($\kappa_\x=L_\x/m_\x$, $\kappa_\y=L_\y/m_\y$)
	$$O\left(\left(\sqrt{\kappa_\x+\kappa_\y}\right)\cdot \ln\left(\kappa_\x\kappa_\y\right)\ln\left(\frac{\kappa_\x\kappa_\y}{\epsilon}\right)\right).$$
	gradient evaluations.
\end{manualtheorem}
\begin{proof}
	Define $\tilde{\x}_{t+1}:=\argmin_\x f(\x,\y_t)$. Let us define $\y^*(\x):=\argmax_\y f(\x,\y)$, $\x^*(\y):=\argmin_\x f(\x,\y)$ and $\phi(\x):=\max_\y f(\x,\y)$. Also define $\tilde\x_{t+1}:=\argmin_\x f(\x,\y^*(\x_t))$ and $\hat\x_{t+1}:=\argmin_\x f(\x,\y_t)$. 
	
	The basic idea is the following. Because $\y^*(\cdot)$ is $L_{\x\y}/m_\y$-Lipschitz and $\x^*(\cdot)$ is $L_{\x\y}/m_\x$-Lipschitz (Fact~\ref{prop:facts}),
	\begin{align*}
	\Vert \x^*(\y_t)-\x^*\Vert &= \Vert \x^*(\y_t)-\x^*(\y^*)\Vert \le \frac{L_{\x\y}}{m_\x} \Vert \y_t-\y^*\Vert,\\
	\Vert \y^*(\x_{t+1})-\y^*\Vert &= \Vert \y^*(\x_{t+1})-\y^*(\x^*)\Vert \le \frac{L_{\x\y}}{m_\y} \Vert \x_{t+1}-\x^*\Vert.
	\end{align*}
	By a standard analysis of accelerated gradient descent (Lemma~\ref{thm:nesterov}), since $\hat\x_{t+1}=\x^*(\y_t)$ is the minimum of $f(\cdot,\y_t)$ and $\x_t$ is the initial point,
	\begin{align*}
	    	\Vert \x_{t+1}-\hat\x_{t+1}\Vert^2&\le (\kappa_\x+1)\Vert \x_t-\hat\x_{t+1}\Vert^2\cdot\left(1-\frac{1}{\sqrt{\kappa_\x}}\right)^{2\sqrt{\kappa_\x}\ln(24\kappa_\x)}\\
	&\le \Vert \x_t-\hat\x_{t+1}\Vert^2\cdot (\kappa_\x+1)\cdot \exp\left\{-2\ln(24\kappa_\x)\right\}\\
	&\le \frac{1}{256}\Vert \x_t-\hat\x_{t+1}\Vert^2.
	\end{align*}
	That is,
	\begin{align*}
	    \Vert \x_{t+1}-\x^*(\y_t)\Vert\le \frac{1}{16}\Vert \x_t-\x^*(\y_t)\Vert \le \frac{1}{16}\left(\Vert \x_t-\x^*\Vert+\Vert \x^*(\y_t)-\x^*\Vert\right).
	\end{align*}
	Thus
	\begin{equation}
	\label{equ:xxcontract}
	    \Vert \x_{t+1}-\x^*\Vert \le \Vert \x_{t+1}-\x^*(\y_t)\Vert + \Vert \x^*(\y_t)-\x^*\Vert\le  \frac{17}{16}\cdot\frac{L_{\x\y}}{m_\x}\Vert \y_t-\y^*\Vert + \frac{1}{16}\Vert \x_t-\x^*\Vert.
	\end{equation}
	Similarly,
	\begin{equation*}
	    \Vert \y_{t+1}-\y^*(\x_{t+1})\Vert \le \frac{1}{16}\Vert \y_t-\y^*(\x_{t+1})\Vert \le \frac{1}{16}\left(\Vert \y_t-\y^*\Vert+\Vert \y^*(\x_{t+1})-\y^*\Vert\right).
	\end{equation*}
	Thus
	\begin{align}
	    \Vert \y_{t+1}-\y^*\Vert &\le \Vert \y_{t+1}-\y^*(\x_{t+1})\Vert + \Vert \y^*(\x_{t+1})-\y^*\Vert \nonumber \\
	    &\le \frac{17}{16}\cdot\frac{L_{\x\y}}{m_\y}\Vert \x_{t+1}-\x^*\Vert + \frac{1}{16}\Vert \y_t-\y^*\Vert \nonumber \\
	    &\le \left(\frac{17^2}{16^2}\cdot\frac{L_{\x\y}^2}{m_\x m_\y}+\frac{1}{16}\right)\Vert \y_t-\y^*\Vert + \frac{17L_{\x\y}}{256 m_\y}\Vert \x_t-\x^*\Vert \nonumber\\
	    &\le 0.35\Vert \y_t-\y^*\Vert + \frac{17L_{\x\y}}{256 m_\y}\Vert \x_t-\x^*\Vert.\label{equ:yycontract}
	\end{align}
	Define $C:=4\sqrt{m_\y/m_\x}$. By adding (\ref{equ:xxcontract}) and $C$ times (\ref{equ:yycontract}), one gets
\begin{align*}
\Vert \x_{t+1}-\x^*\Vert + C\Vert \y_{t+1}-\y^*\Vert &\le \left(\frac{1}{16}+\frac{17L_{\x\y}}{64\sqrt{m_\x m_\y}}\right)\Vert \x_t-\x^*\Vert + \left(0.35C+\frac{17}{16}\cdot\frac{L_{\x\y}}{m_\x}\right)\Vert \y_t-\y^*\Vert\\
&\le \frac{1}{2}\Vert \x_t-\x^*\Vert + \left(0.35+\frac{17L_{\x\y}}{64\sqrt{m_\x m_\y}}\right)C\Vert \y_t-\y^*\Vert\\
&\le \frac{1}{2}\left(\Vert\x_t-\x^*\Vert+C\Vert\y_t-\y^*\Vert\right).
\end{align*}
	It follows that
	\begin{align*}
	 \Vert \x_{T}-\x^*\Vert + C\Vert \y_{T}-\y^*\Vert \le 2^{-T}\left(\Vert \x_0-\x^*\Vert + C\Vert\y_0-\y^*\Vert\right).
	\end{align*}
If $C\ge 1$, then
\begin{align*}
    \Vert \x_T-\x^*\Vert + \Vert \y_T-\y^*\Vert \le 4\sqrt{\frac{m_\y}{m_\x}}\cdot 2^{-T}\cdot\left(\Vert \x_0-\x^*\Vert+\Vert\y_0-\y^*\Vert\right).
\end{align*}
On the other hand, if $C<1$, then
\begin{align*}
    \Vert \x_T-\x^*\Vert + \Vert \y_T-\y^*\Vert \le \frac{2^{-T}}{C}\left(\Vert \x_0-\x^*\Vert+\Vert\y_0-\y^*\Vert\right)=\sqrt{\frac{m_\x}{m_\y}}\cdot 2^{-T-1}\left(\Vert \x_0-\x^*\Vert+\Vert\y_0-\y^*\Vert\right).
\end{align*}
Since $\max\{m_\x/m_\y, m_\y/m_\x\} \le L_\x / \min\{m_\x,m_\y\}$,
\begin{equation}
    \Vert \x_T-\x^*\Vert + \Vert \y_T-\y^*\Vert \le 4\sqrt{\frac{L_\x}{\min\{m_\x,m_\y\}}}\cdot 2^{-T}\left(\Vert \x_0-\x^*\Vert+\Vert\y_0-\y^*\Vert\right).
\end{equation}

	The theorem follows from this inequality.
\end{proof}

\section{Proof of Theorem 2}
\label{append:thm2proof}
\begin{manualtheorem}{2}
Assume that $M\ge 20\kappa\sqrt{2\kappa+\frac{L}{m_\x}+\frac{L_{\x\y}^2}{m_\x m_\y}}\left(1+\frac{L}{m_\y}\right).$
The number of iterations needed by Algorithm~2 to produce $(\x_T,\y_T)$ such that 
$$\Vert \x_T-\x^*\Vert+\Vert \y_T-\y^*\Vert \le \epsilon\left(\Vert \x_0-\x^*\Vert+\Vert \y_0-\y^*\Vert \right)$$
is at most ($\kappa=\beta/m_\x$)
\begin{equation}
\hat T=8\sqrt{\kappa}\cdot\ln\left(\frac{28\kappa^2 L}{m_\y}\sqrt{\frac{L^2}{m_\x m_\y}}\cdot\frac{1}{\epsilon}\right).
\end{equation}
\end{manualtheorem}

Before proving the theorem, we would first state the inexact accelerated proximal point algorithm~\cite{lin2020near}, which is the basis of Algorithm 2.

\begin{algorithm}
\renewcommand{\thealgorithm}{II}
	\caption{Inexact Accelerated Proximal Point Algorithm (Inexact APPA)}
	\label{algo:inexactappa}
	\begin{algorithmic}
		\Require Initial point $\x_0$, proximal parameter $\beta$, strongly convex module $m$
		\State $\hat\x_0\gets \x_0$, $\kappa\gets \beta/m$, $\theta\gets\frac{2\sqrt{\kappa}-1}{2\sqrt{\kappa}+1}$, $\tau\gets \frac{1}{2\sqrt{\kappa}+4\kappa}$
		\For{$t=1,\cdots,T$}
			\State Find $\x_t$ such that $g(\x_t)+\beta\Vert\x_t-\hat\x_{t-1}\Vert^2\le \min_\x \{g(\x)+\beta\Vert\x-\hat\x_{t-1}\Vert^2\}+\delta_t$
			\State $\hat\x_t\gets \x_t+\theta(\x_t-\x_{t-1})+\tau(\x_t-\hat\x_{t-1})$
		\EndFor
	\end{algorithmic}
\end{algorithm}

The following two lemmas about the inexact APPA algorithm follow from the proof of Theorem 4.1~\cite{lin2020near} in an earlier version of the paper. Here we provide their proofs for completeness.

\begin{lemma}
\label{lem:b1}
Suppose that $\{(\x_t,\hat\x_t)\}_{t\ge 0}$ are generated by running the inexact APPA algorithm on $g(\cdot)$. Then $\forall t\ge1, \forall\x$,
$$g(\x)\ge g(\x_t)-2\beta(\x-\x_t)^T(\x_t-\hat\x_{t-1})+\frac{m}{4}\Vert\x-\x_t\Vert^2-7\kappa\delta_t.$$
\end{lemma}
\begin{proof}[Proof of Lemma~\ref{lem:b1}]
By definition
$$g(\x_t)+\beta\Vert\x_t-\hat\x_{t-1}\Vert\le \min_\x\{g(\x)+\beta\Vert\x-\hat\x_{t-1}\Vert^2\}+\delta_t.$$
Define $\x^*_t:=\argmin_{\x}\{g(\x)+\beta\Vert\x-\hat\x_{t-1}\Vert^2\}$. By the $m$-strong convexity of $g(\cdot)$, we have $\forall \x$,
$$g(\x)+\beta\Vert\x-\hat\x_{t-1}\Vert^2\ge g(\x^*_t)+\beta\Vert\x^*_t-\hat\x_{t-1}\Vert^2+\left(\frac{m}{2}+\beta\right)\Vert\x-\x_t^*\Vert^2.$$
Equivalently,
\begin{align*}
g(\x)&\ge g(\x_t)+\beta\Vert\x_t-\hat\x_{t-1}\Vert^2-\beta\Vert\x-\hat\x_{t-1}\Vert^2+\left(\beta+\frac{m}{2}\right)\Vert\x-\x^*_t\Vert^2-\delta_t\\
&= g(\x_t)-2\beta(\x-\x_t)^T(\x_t-\hat\x_{t-1})-\beta\Vert\x-\x_t\Vert^2+\left(\beta+\frac{m}{2}\right)\Vert\x-\x^*_t\Vert^2-\delta_t.
\end{align*}
On the other hand, we have
\begin{align*}
\left(\beta+\frac{m}{2}\right)\Vert\x-\x^*_t\Vert^2-\beta\Vert\x-\x_t\Vert^2 = \frac{m\Vert\x-\x_t\Vert^2}{2}+(2\beta+m)(\x-\x_t)^T(\x_t-\x^*_t)+(\beta+\frac{m}{2})\Vert\x_t-\x^*_t\Vert^2.
\end{align*}
By Cauchy-Schwarz Inequality,
\begin{equation*}
(\x-\x_t)^T(\x_t-\x^*_t)\ge -\frac{m\Vert\x-\x_t\Vert^2}{4(2\beta+m)}-(1+2\kappa)\Vert\x_t-\x^*_t\Vert^2.
\end{equation*}
Putting the pieces together yields
\begin{align*}
g(\x)&\ge g(\x_t)-2\beta(\x-\x_t)^T(\x_t-\hat\x_{t-1})+\frac{m\Vert\x-\x_t\Vert^2}{2}+(\beta+\frac{m}{2})\Vert\x_t-\x^*_t\Vert^2\\
&\quad -\frac{m\Vert\x-\x_t\Vert^2}{4}-(1+2\kappa)(2\beta+m)\Vert\x_t-\x^*_t\Vert^2-\delta_t\\
&=g(\x_t) - 2\beta(\x-\x_t)^T(\x_t-\hat\x_{t-1}) + \frac{m\Vert\x-\x_t\Vert^2}{4} - (2\beta+m)(\frac{1}{2}+2\kappa)\Vert\x_t-\x^*_t\Vert^2-\delta_t.
\end{align*}
Also, since $g(\x)+\beta\Vert\x-\hat\x_{t-1}\Vert^2$ is $(2\beta+m)$-strongly convex, 
\begin{equation*}
	\Vert \x_t-\x^*_t\Vert^2 \le \frac{2}{2\beta+m}\left(g(\x_t)+\beta\Vert\x_t-\hat\x_{t-1}\Vert^2-\min_{\x}\{g(\x)+\beta\Vert\x-\hat\x_{t-1}\Vert^2\}\right)\le \frac{2\delta_t}{2\beta+m}.
\end{equation*}
Thus
\begin{align*}
g(\x)\ge g(\x_t)- 2\beta(\x-\x_t)^T(\x_t-\hat\x_{t-1}) + \frac{m\Vert\x-\x_t\Vert^2}{4} - 7\kappa\delta_t.
\end{align*}
\end{proof}

\begin{lemma}
	\label{lem:appamain}
	Suppose that $\{\x_t\}_{t\ge 0}$ is generated by running the inexact APPA algorithm on $g(\cdot)$. There exists a sequence $\{\Lambda_t\}_{t\ge 0}$ such that
	\begin{enumerate}
		\item $\Lambda_t\ge g(\x_t)$
		\item $\Lambda_0-g(\x^*)\le 2(g(\x_0)-g(\x^*))$
		\item $\Lambda_{t+1}-g(\x^*)\le \left(1-\frac{1}{2\sqrt{\kappa}}\right)\left(\Lambda_t-g(\x^*)\right)+11\kappa\delta_{t+1}$
	\end{enumerate}
\end{lemma}
\begin{proof}[Proof of Lemma~\ref{lem:appamain}]
Let us slightly abuse notation, and define a sequence of functions $\{\Lambda(\x)\}_{t\ge 0}$ first:
\begin{align*}
\Lambda_0(\x)&:=g(\x_0)+\frac{m\Vert\x-\x_0\Vert^2}{4},\\
\Lambda_{t+1}(\x)&:=\frac{1}{2\sqrt{\kappa}}\left(g(\x_{t+1})+2\beta(\hat\x_t-\x_{t+1})^T(\x-\x_{t+1})+\frac{m\Vert\x-\x_{t+1}\Vert^2}{4}+14\kappa^{3/2}\delta_{t+1}\right)+\left(1-\frac{1}{2\sqrt{\kappa}}\right)\Lambda_t(\x).
\end{align*}
The sequence $\{\Lambda_t\}_{t\ge 0}$ in the lemma is then defined as $\Lambda_t:=\Lambda_t(\x^*)$. Note that later we do not need to make use of the explicit definition of $\Lambda_t$.

From the definition, Property 2 is straightforward, as
\begin{align*}
\Lambda_0-g(\x^*) = \frac{m\Vert\x^*-\x_0\Vert^2}{4}+g(\x_0)-g(\x^*) \le \frac{1}{2}(g(\x_0)-g(\x^*))+g(\x_0)-g(\x^*).
\end{align*}

Now, let us show $\Lambda_t\ge \min_{\x}\Lambda_t(\x)\ge g(\x_t)$ using induction. Let $\w_t:=\argmin_{\x}\Lambda_t(\x)$ and $\Lambda_t^*:=\min_\x \Lambda_t(\x)$. Observe that $\Lambda_t(\x)$ is always a quadratic function of the form $\Lambda_t(\x)=\Lambda^*_t+\frac{m}{4}\Vert\x-\w_t\Vert^2$. Then the following recursions hold for $\w_t$ and $\Lambda_t^*$:
\begin{align}
\w_{t+1}&=(1-\frac{1}{2\sqrt{\kappa}})\w_t+2\sqrt{\kappa}(\x_{t+1}-\hat\x_t)+\frac{\x_{t+1}}{2\sqrt{\kappa}}, \nonumber \\ 
\Lambda_{t+1}^* &= \left(1-\frac{1}{2\sqrt{\kappa}}\right)\Lambda^*_t+\frac{1}{2\sqrt{\kappa}}\left(g(\x_{t+1})+14\kappa^{3/2}\delta_{t+1}\right) \nonumber \\
&\quad + \frac{1}{2\sqrt{\kappa}}\left(1-\frac{1}{2\sqrt{\kappa}}\right)\left(\frac{m\Vert\x_{t+1}-\w_t\Vert^2}{4}+2\beta(\hat\x_t-\x_{t+1})^T(\w_t-\x_{t+1})\right). \nonumber
\end{align}
The recursion for $\w_{t+1}$ can be derived by differentiating both sides in the recusion of $\Lambda_t(\x)$, while the recursion for $\Lambda^*_{t+1}$ can be derived by plugging the recursion for $\w_{t+1}$ into $\Lambda^*_{t+1}=\Lambda_{t+1}(\w_{t+1})$.

Now, assume that $\Lambda_t^*\ge g(\x_t)$ for $t\le T-1$. Then
\begin{align}
\Lambda^*_T&\ge \left(1-\frac{1}{2\sqrt{\kappa}}\right)g(\x_{T-1})+\frac{1}{2\sqrt{\kappa}}\left(g(\x_{T})+14\kappa^{3/2}\delta_{T}\right) \nonumber \\
&\quad + \frac{1}{2\sqrt{\kappa}}\left(1-\frac{1}{2\sqrt{\kappa}}\right)\left(\frac{m\Vert\x_{T}-\w_{T-1}\Vert^2}{4}+2\beta(\hat\x_{T-1}-\x_{T})^T(\w_{T-1}-\x_{T})\right).
\label{equ:step1}
\end{align}

Applying Lemma~\ref{lem:b1} with $\x=\x_{T-1}$ yields
\begin{equation}
\label{equ:step2}
g(\x_{T-1}) \ge g(\x_T)+2\beta(\x_{T-1}-\x_T)^T(\hat\x_{T-1}-\x_T)+\frac{m\Vert\x_{T-1}-\x_T\Vert^2}{4}-7\kappa\delta_T.
\end{equation}
 Summing (\ref{equ:step1}) and (\ref{equ:step2}) gives
 \begin{align*}
 \Lambda^*_T&\ge g(\x_T)+2\beta(1-\frac{1}{2\sqrt{\kappa}})(\hat\x_{T-1}-\x_T)^T\left[(\x_{T-1}-\x_T)+\frac{\w_{T-1}-\x_T}{2\sqrt{\kappa}}\right]\\
 &\ge g(\x_T)  +2\beta(1-\frac{1}{2\sqrt{\kappa}})(\hat\x_{T-1}-\x_T)^T\left[(\x_{T-1}-\hat\x_{T-1})+\frac{\w_{T-1}-\hat\x_{T-1}}{2\sqrt{\kappa}}\right].
 \end{align*}
The second inequality follows from
$$(\hat\x_{T-1}-\x_T)^T\left(\x_T-\hat\x_{T-1}+\frac{\x_T-\hat\x_{T-1}}{2\sqrt{\kappa}}\right)\le 0.$$
By the update formula
$$\hat\x_{t+1}=\x_{t+1}+\frac{2\sqrt{\kappa}-1}{2\sqrt{\kappa}+1}(\x_{t+1}-\x_t)+\frac{1}{2\sqrt{\kappa}+4\kappa}(\x_{t+1}-\hat\x_t)$$ 
and the recursive rule for $\w_t$, we get
\begin{align*}
&(\x_{t+1}-\hat\x_{t+1})+\frac{1}{2\sqrt{\kappa}}(\w_{t+1}-\hat\x_{t+1})\\
=&\x_{t+1}+\frac{1}{2\sqrt{\kappa}}\left[\left(1-\frac{1}{2\sqrt{\kappa}}\w_t+2\sqrt{\kappa}(\x_{t+1}-\hat\x_t)+\frac{\x_{t+1}}{2\sqrt{\kappa}}\right)\right]\\
&\quad -\left(1+\frac{2}{\sqrt{\kappa}}\right)\left(\x_{t+1}+\frac{2\sqrt{\kappa}-1}{2\sqrt{\kappa}+1}(\x_{t+1}-\x_t)+\frac{1}{2\sqrt{\kappa}+4\kappa}(\x_{t+1}-\hat\x_t)\right)\\
=&\left(1-\frac{2}{\sqrt{\kappa}}\right)\left[\x_t+\frac{1}{2\sqrt{\kappa}}\w_t-\left(1+\frac{1}{2\sqrt{\kappa}}\hat\x_t\right)\right].
\end{align*}
Meanwhile, when $t=0$, $\x_t=\hat\x_t=\w_t=\x_0$. Thus, by induction, we have for any $t$, $(\x_t-\hat\x_t)+\frac{1}{2\sqrt{\kappa}}(\w_t-\hat\x_t)=0.$ As a result $\Lambda^*_T\ge g(\x_T)$. Again, by induction, this holds for all $T$. This proves Property 1 in the lemma.

Let us now focus on the final property. Combining Lemma~\ref{lem:b1} and the recursion for $\Lambda_t(\x)$,
\begin{align*}
\Lambda_{t+1}=\Lambda_{t+1}(\x^*)\le \left(1-\frac{1}{2\sqrt{\kappa}}\right)\Lambda_t+\frac{1}{2\sqrt{\kappa}}\left(g(\x^*)+14\kappa^{3/2}\delta_{t+1}+7\kappa\delta_{t+1}\right).
\end{align*}
It follows that
\begin{align*}
\Lambda_{t+1}-g(\x^*)\le \left(1-\frac{1}{2\sqrt{\kappa}}\right)(\Lambda_t-g(\x^*))+11\kappa\delta_{t+1}.
\end{align*}
This is exactly Property 3.
\end{proof}

Now we are ready to prove Theorem~\ref{thm:appa}. 

\begin{proof}
	Define $\phi(\x):=\max_\y f(\x,\y)$ and $\hat{L}:=L+L_{\x\y}^2/m_\y$. Then $\phi(\x)$ is $m_\x$-strongly convex and $\hat{L}$-smooth. Observe that
	\begin{align*}
	\x^*_t&=\argmin_\x\left[\phi(\x)+\beta\Vert\x-\hat\x_{t-1}\Vert^2\right],\\
	\y^*_t&=\argmax_\y\left[f(\x^*_t,\y)\right].
	\end{align*}
	Thus Algorithm 2 is an instance of the inexact APPA algorithm on $\phi(\x)$ with proximal parameter $\beta$ and strongly convex module $m_\x$, and with
	\begin{align}
	\delta_t&=\phi(\x_t)+\beta\Vert\x_t-\hat\x_{t-1}\Vert^2-\min_\x\left\{\phi(\x)+\beta\Vert\x-\hat\x_{t-1}\Vert^2\right\}\nonumber \\ \label{equ:deltat}
	&\le \frac{\hat L+2\beta}{2}\Vert \x_t-\x^*_t\Vert^2.
	\end{align}
	Here we used the fact that, for a $L$-smooth function $g(\cdot)$ whose minimum is $\x^*$, 	$g(\x)-g(\x^*)\le \frac{L}{2}\Vert \x-\x^*\Vert^2$.
	Define $C_1:=\Vert \x_0-\x^*\Vert+\Vert\y_0-\y^*\Vert$ and 
	$C_0:=44\kappa\sqrt{\kappa}\frac{\hat L+2\beta}{2}C_1^2$. Let us state the following induction hypothesis
	\begin{equation}
	\label{equ:valuehypo}
	\Delta_t:=\Lambda_t-\phi(\x^*)\le C_0\left(1-\frac{1}{4\sqrt{\kappa}}\right)^t,
	\end{equation}
	\begin{equation}
	\label{equ:distancehypo}
	\epsilon_t:=\Vert \x_t-\x^*_t\Vert+\Vert\y_t-\y^*_t\Vert\le C_1\left(1-\frac{1}{4\sqrt{\kappa}}\right)^{\frac{t}{2}}. 
	\end{equation}
	It is easy to verify that with our choice of $C_0$ and $C_1$, both (\ref{equ:valuehypo}) and (\ref{equ:distancehypo}) hold for $t=0$.
	
	Now, assume that (\ref{equ:valuehypo}) and (\ref{equ:distancehypo}) hold for $\tau=1,2,\cdots,t$. Define $\y^*(\cdot):=\argmax_\y f(\cdot,\y)$. By Fact~\ref{prop:facts}, $\y^*(\cdot)$ is $(L/m_\y)$-Lipschitz. Thus
	\begin{align*}
	\Vert \y_t-\y^*_{t+1}\Vert &\le \Vert \y^*_t-\y^*_{t+1}\Vert+\Vert \y_t-\y^*_t\Vert\\
	&\le \Vert \y^*(\x^*_t)-\y^*(\x^*_{t+1})\Vert + \epsilon_t\\
	&\le \frac{L}{m_\y}\cdot\left(\Vert \x^*_t-\x_t\Vert+\Vert\x_t-\x^*_{t+1}\Vert \right)+\epsilon_t\\
	&\le \left(\frac{L}{m_\y}+1\right)\epsilon_t+\frac{L}{m_\y}\Vert\x_t-\x^*_{t+1}\Vert.
	\end{align*}
	It follows that
	\begin{equation}
	\label{equ:equ6}
	\epsilon_{t+1}\le \frac{1}{M}\left[\Vert \x_t-\x^*_{t+1}\Vert+\Vert\y_t-\y^*_{t+1}\Vert\right]\le \frac{1+\frac{L}{m_\y}}{M}\cdot\left(\Vert\x_t-\x^*_{t+1}\Vert+\epsilon_t\right).
	\end{equation}
	
	
	Note that by Lemma~\ref{lem:appamain} and the induction hypothesis (\ref{equ:valuehypo})
	\begin{align*}
	    \phi(\x^*_{t+1})-\phi(\x^*)\le \left(1-\frac{1}{2\sqrt{\kappa}}\right)\Delta_t\le C_0\left(1-\frac{1}{4\sqrt{\kappa}}\right)^t.
	\end{align*}
	By the $m_\x$-strong convexity of $\phi(\cdot)$ (Fact~\ref{prop:facts}),
	\begin{align*}
	    \Vert \x^*_{t+1}-\x^*\Vert \le \sqrt{\frac{2}{m_\x}\left(\phi(\x^*_{t+1})-\phi(\x^*)\right)}\le \sqrt{\frac{2C_0}{m_\x }}\left(1-\frac{1}{4\sqrt{\kappa}}\right)^{\frac{t}{2}}.
	\end{align*}
	Meanwhile
	\begin{align*}
	    \Vert \x_t-\x^*\Vert \le \sqrt{\frac{2}{m_\x}\left(\phi(\x_t)-\phi(\x^*)\right)} \le \sqrt{\frac{2C_0}{m_\x }}\left(1-\frac{1}{4\sqrt{\kappa}}\right)^{\frac{t}{2}}.
	\end{align*}
	Therefore
	\begin{equation}
	    \Vert \x_t-\x^*_{t+1}\Vert \le \Vert\x_t-\x^*\Vert+\Vert \x^*_{t+1}-\x^*\Vert \le 2\sqrt{\frac{2C_0}{m_\x}}\left(1-\frac{1}{4\sqrt{\kappa}}\right)^{\frac{t}{2}}.
	\end{equation}
	By (\ref{equ:equ6}), (\ref{equ:distancehypo}) and the fact that $M\ge 20\kappa\sqrt{2\kappa+\frac{\hat L}{m_\x}}(1+L/m_\y)$
    \begin{align*}
        \epsilon_{t+1}&\le \frac{1+\frac{L}{m_\y}}{M}\left(2\sqrt{\frac{2C_0}{m_\x}}+C_1\right)\left(1-\frac{1}{4\sqrt{\kappa}}\right)^{\frac{t}{2}}\\
        &\le \frac{1+\frac{L}{m_\y}}{M}\cdot\left(1+2\sqrt{\frac{44\kappa^{1.5}(\hat L+2\beta)}{m_\x}}\right)C_1\left(1-\frac{1}{4\sqrt{\kappa}}\right)^{\frac{t}{2}} \tag{$C_0=44\kappa^{1.5}\frac{\hat L+2\beta}{2}C_1^2$}\\
        &\le \frac{1+2\sqrt{44}\kappa\sqrt{\frac{\hat L+2\beta }{m_\x}}}{20\kappa\sqrt{2\kappa+\frac{\hat L}{m_\x}}}\cdot C_1\left(1-\frac{1}{4\sqrt{\kappa}}\right)^{\frac{t}{2}} \tag{$2\sqrt{44}+1<15$}\\
        &\le \frac{3}{4}C_1\left(1-\frac{1}{4\sqrt{\kappa}}\right)^{\frac{t}{2}}\le C_1\left(1-\frac{1}{4\sqrt{\kappa}}\right)^{\frac{t+1}{2}}.
    \end{align*}	
	Therefore (\ref{equ:distancehypo}) holds for $t+1$. Meanwhile, by (\ref{equ:deltat}) and Lemma~\ref{lem:appamain},
	\begin{align*}
	\Delta_{t+1}&\le \left(1-\frac{1}{2\sqrt{\kappa}}\right)\Delta_t+11\kappa\cdot\frac{\hat L+2\beta}{2}\epsilon_{t+1}^2\\
	&\le \left(1-\frac{1}{2\sqrt{\kappa}}\right)C_0\left(1-\frac{1}{4\sqrt{\kappa}}\right)^t+11\kappa\cdot\frac{\hat L+2\beta}{2}\cdot C_1^2\left(1-\frac{1}{4\sqrt{\kappa}}\right)^t\\
	&=C_0\left(1-\frac{1}{4\sqrt{\kappa}}\right)^{t+1},
	\end{align*}
	where we used the fact that 
	\begin{align*}
	11\kappa\cdot\frac{\hat L+2\beta}{2}\cdot C_1^2= \frac{1}{4\sqrt{\kappa}}\cdot 44\kappa^{1.5}\frac{\hat L+2\beta}{2}C_1^2=\frac{C_0}{4\sqrt{\kappa}}.
	\end{align*}
	Thus (\ref{equ:valuehypo}) also holds for $t+1$. By induction on $t$, we can see that (\ref{equ:valuehypo}) and (\ref{equ:distancehypo}) both hold for all $t\ge 0$.
	
	As a result,
	\begin{align*}
		\Vert \x_T-\x^*\Vert &\le \sqrt{\frac{2}{m_\x}\left[\phi(\x_T)-\phi(\x^*)\right]} \le \sqrt{\frac{2}{m_\x}\cdot 44\kappa\sqrt{\kappa}\frac{\hat L + 2\beta}{2}C_1^2\left(1-\frac{1}{4\sqrt{\kappa}}\right)^{\frac{T}{2}}}\\
		&\le C_1\left(1-\frac{1}{4\sqrt{\kappa}}\right)^{\frac{T}{2}}\sqrt{88\kappa\sqrt{\kappa}\cdot\left(\frac{L^2}{m_\x m_\y}+\kappa\right)}.
	\end{align*}
	Meanwhile,
	\begin{align*}
		\Vert \y_T-\y^*\Vert &\le \Vert \y_T-\y^*(\x_T)\Vert + \Vert\y^*-\y^*(\x_T)\Vert\le \epsilon_T+\frac{L_{\x\y}}{m_\y}\Vert \x_T-\x^*\Vert.
	\end{align*}
	Therefore
	\begin{align*}
		\Vert \x_T-\x^*\Vert +\Vert \y_T-\y^*\Vert &\le \epsilon_T+ \left(\frac{L_{\x\y}}{m_\y}+1\right)\Vert \x_T-\x^*\Vert\\
		&\le C_1\left(1-\frac{1}{4\sqrt{\kappa}}\right)^{\frac{T}{2}}+\frac{2L}{m_\y}\cdot C_1\left(1-\frac{1}{4\sqrt{\kappa}}\right)^{\frac{T}{2}}\cdot \sqrt{88\kappa\sqrt{\kappa}\cdot\left(\frac{L^2}{m_\x m_\y}+\kappa\right)}\\
		&\le C_1\left(1-\frac{1}{4\sqrt{\kappa}}\right)^{\frac{T}{2}}\cdot\left[1+\frac{27\kappa^2 L}{m_\y}\sqrt{\frac{L^2}{m_\x m_\y}}\right]\\
		&\le \frac{28\kappa^2 L}{m_\y}\sqrt{\frac{L^2}{m_\x m_\y}}\cdot \left(1-\frac{1}{4\sqrt{\kappa}}\right)^{\frac{T}{2}}\cdot\left(\Vert \x_0-\x^*\Vert+\Vert\y_0-\y^*\Vert\right),
	\end{align*}
	which proves the theorem.

\end{proof}

\section{Proof of Theorem 3}
\label{append:thm3proof}
\begin{manualtheorem}{3}
    Assume that $f\in\mathcal{F}(m_\x,m_\y,L_\x,L_{\x\y},L_\y)$. In Algorithm~4, the gradient complexity to produce $(\x_T,\y_T)$ such that $\Vert \z_T-\z^*\Vert\le \epsilon$ is
\begin{equation*}
O\left(\sqrt{\frac{L_\x}{m_\x}+\frac{L\cdot L_{\x\y}}{m_\x m_\y}+\frac{L_\y}{m_\y}}\cdot\ln^3\left(\frac{L^2}{m_\x m_\y}\right)\ln\left(\frac{L^2}{m_\x m_\y}\cdot\frac{\Vert\z_0-\z^*\Vert}{\epsilon}\right)\right).
\end{equation*}
\end{manualtheorem}

\begin{proof}
We start the proof by verifying $f(\x,\y)+\beta_1\Vert \x-\hat\x\Vert^2-\beta_2\Vert \y-\hat\y\Vert^2$ can indeed be solved by calling ABR($\cdot$,$[\x_0;\y_0]$,$1/M_2$, $2\beta_1$, $2\beta_2$, $3L$, $3L$). Observe that $L_{\x\y}\le\beta_1,\beta_2\le L$. Since $f(\x,\y)+\beta_1\Vert \x-\hat\x\Vert^2-\beta_2\Vert \y-\hat\y\Vert^2$ is $2\beta_1$-strongly convex w.r.t. $\x$ and $2\beta_2$-strongly concave w.r.t. $\y$, we can see that $\frac{1}{2}\sqrt{2\beta_1\cdot 2\beta_2}\ge L_{\x\y}$. We can also verify that $f(\x,\y)+\beta_1\Vert \x-\hat\x\Vert^2-\beta_2\Vert \y-\hat\y\Vert^2$ is $3L$-smooth, which follows from the fact that $L+\max\{2\beta_1,2\beta_2\}\le 3L$.

Therefore, we can apply Theorem~\ref{thm:br} and conclude that at line $5$ of Algorithm~3
\begin{equation*}
\Vert \x_t-\x^*_t\Vert + \Vert \y_t-\y^*_t\Vert \le \frac{1}{M_2}\left(\Vert \x_{t-1}-\x^*_t\Vert + \Vert \y_{t-1}-\y^*_t\Vert\right),
\end{equation*}
where $(\x^*_t,\y^*_t):=\min_\x\max_\y\{g(\x,\y)-\beta_2\Vert\y-\y_{t-1}\Vert^2\}$, \footnote{Here $g(\x,\y)$ refers to the argument passed to Algorithm~3, which in our case has the form $f(\x,\y)+\beta\Vert\x-\hat\x_{t'-1}\Vert^2$.}
and such $(\x_t,\y_t)$ is found in a gradient complexity of 
\begin{equation*}
\label{equ:iterbound3}
O\left(\sqrt{\frac{L}{\beta_1}+\frac{L}{\beta_2}}\cdot\ln\left(\frac{L^2}{\beta_1\beta_2}\right)\ln\left(\frac{L^2}{\beta_1\beta_2}\cdot M_2\right)\right)=O\left(\sqrt{\frac{L}{\beta_1}+\frac{L}{\beta_2}}\cdot\ln^2\left(\frac{L^2}{m_\x m_\y}\right)\right).
\end{equation*}

Next, we verify that Algorithm~3 is an instance of Algorithm~2 on the function $\hat{g}(\x,\y):=-g(\y,\x).$ Notice that
$$\min_\y \max_\x \left\{-g(\x,\y)+\beta\Vert\y-\hat\y\Vert^2\right\}=-\min_\x\max_\y\left\{g(\x,\y)-\beta\Vert\y-\hat\y\Vert^2\right\}.$$
That is, $\min_\x \max_\y\left\{g(\x,\y)-\Vert\y-\hat\y\Vert^2\right\}$ has the same saddle point as $-g(\x,\y)+\beta\Vert\y-\hat\y\Vert^2$. Thus, we only need to verify that
\begin{equation}
\label{equ:m2satisfy}
M_2\ge 20\cdot \frac{\beta_2}{m'_\y}\left(1+\frac{L'}{m'_\x}\right)\sqrt{\frac{2\beta_2}{m'_\y}+\frac{L'}{m'_\y}+\frac{L_{\x\y}^2}{m_\x m_\y}},
\end{equation}
where $(m'_\x,m'_\y,L'_\x,L_{\x\y},L'_\y)$ are parameters for $f(\x,\y)+\beta_1\Vert \x\Vert^2$, and $L'=\max\{L_{\x\y},L'_\x,L'_\y\}$. Note that $m'_\x \ge m_\x + 2\beta_1$, $m'_\y=m_\y$, $L'_\x=L'_\y\le L+2\beta_1$, $L_{\x\y}\le \beta_1,\beta_2\le L$. Thus
\begin{align*}
\text{RHS of (\ref{equ:m2satisfy})} &\le 20\cdot\frac{\beta_2}{m_\y}\sqrt{\frac{2\beta_2}{m'_\y}+\frac{L+2\beta_1}{m_\y}+\frac{L_{\x\y}^2}{m_\y (m_\x+2\beta_1)}}\cdot\left(1+\frac{L+2\beta_1}{m_\x+2\beta_1}\right)\\
&\le 20\cdot\frac{L}{m_\y}\sqrt{\frac{2L}{m_\y}+\frac{3L}{m_\y}+\frac{L_{\x\y}}{2m_\y}}\left(1+\frac{L}{m_\x}\right)\\
&\le \frac{96L^{2.5}}{m_\x m_\y^{1.5}}=M_2.
\end{align*}
Therefore, Algorithm~3 is indeed an instance of Inexact APPA (Algorithm~\ref{algo:inexactappa}). Notice that by the stopping condition of Algorithm~3,
\begin{align*}
\left(\Vert \x_t-\x^*\Vert+\Vert\y_t-\y^*\Vert\right)&\le \frac{\sqrt{2}}{\min\{m_\x,m_\y\}}\Vert \nabla g(\x_t,\y_t)\Vert \tag{Fact \ref{fact:gradnorm} and \ref{fact:z}}\\
&\le \frac{\sqrt{2}}{\min\{m_\x,m_\y\}}\cdot \frac{\min\{m_\x,m_\y\}}{9LM_1}\Vert \nabla g(\x_0,\y_0)\Vert\\
&\le \frac{\sqrt{2}}{\min\{m_\x,m_\y\}}\cdot \frac{\min\{m_\x,m_\y\}}{9LM_1}\cdot 6L\left(\Vert\x_0-\x^*\Vert+\Vert\y_0-\y^*\Vert\right)\\
&\le \frac{1}{M_1}\left(\Vert \x_0-\x^*\Vert+\Vert \y_0-\y^*\Vert\right).
\end{align*}
Thus when Algorithm~3 returns,
\begin{equation}
\label{equ:m1good}
\Vert \x_t-\x^*\Vert+\Vert\y_t-\y^*\Vert\le \frac{1}{M_1}\left(\Vert \x_0-\x^*\Vert+\Vert \y_0-\y^*\Vert\right)
\end{equation}
On the other hand, suppose that
\begin{equation*}
\Vert \x_t-\x^*\Vert+\Vert\y_t-\y^*\Vert\le \frac{1}{M_1}\frac{\min\{m_\x,m_\y\}}{12L}\cdot\left(\Vert \x_0-\x^*\Vert+\Vert \y_0-\y^*\Vert\right),
\end{equation*}
we can show that
\begin{align*}
\Vert \nabla g(\x_t,\y_t)\Vert &\le 6L \left(\Vert \x_t-\x^*\Vert+\Vert\y_t-\y^*\Vert\right)\\
&\le \frac{\min\{m_\x,m_\y\}}{2M_1}\left(\Vert \x_0-\x^*\Vert+\Vert\y_0-\y^*\Vert\right)\\
&\le \frac{1}{M_1}\Vert \nabla g(\x_0,\y_0)\Vert.
\end{align*}
Thus in this case Algorithm~3 must return. By Theorem~\ref{thm:appa}, we can see that Algorithm~3 always returns in at most
\begin{equation}
\label{equ:iterbound2}
O\left(\sqrt{\frac{\beta_2}{m_\y}}\cdot\ln\left(\frac{L^2}{m_\x m_\y}\cdot \frac{12L}{\min\{m_\x,m_\y\}}M_1\right)\right)=O\left(\sqrt{\frac{\beta_2}{m_\y}}\cdot\ln\left(\frac{L^2}{m_\x m_\y}\right)\right)
\end{equation}
iterations.

Finally, we verify that Algorithm~4 is an instance of Algorithm~2 on $f(\x,\y)$ with parameter $\beta_1$. Note that by (\ref{equ:m1good}), we only need to verify that
\begin{equation*}
    M_1=\frac{80L^3}{m_\x^{1.5}m_\y^{1.5}}\ge 20\cdot\frac{\beta_1}{m_\x}\sqrt{\frac{2\beta_1}{m_\x}+\frac{L}{m_\x}+\frac{L_{\x\y}^2}{m_\x m_\y}}\left(1+\frac{L}{m_\y}\right).
\end{equation*}
Observe that
\begin{align*}
    20\cdot\frac{\beta_1}{m_\x}\sqrt{\frac{2\beta_1}{m_\x}+\frac{L}{m_\x}+\frac{L_{\x\y}^2}{m_\x m_\y}}\left(1+\frac{L}{m_\y}\right)&\le  20\cdot\frac{L}{m_\x}\sqrt{\frac{2L}{m_\x}+\frac{L}{m_\x}+\frac{L^2}{m_\x m_\y}}\cdot \frac{2L}{m_\y}\\
    &\le  20\cdot \frac{L}{m_\x}\cdot\sqrt{\frac{4L^2}{m_\x m_\y}} \cdot\frac{2L}{m_\y}=M_1.
\end{align*}
Therefore Algorithm~4 is indeed an instance of Algorithm~2 on $f(\x,\y)$. As a result, by Theorem~\ref{thm:appa}, the number of iterations needed such that $\Vert \z_T-\z^*\Vert \le \epsilon$ is
\begin{equation}
\label{equ:iterbound1}
O\left(\sqrt{\frac{\beta_1}{m_\x}}\cdot\ln\left(\frac{L^2}{m_\x m_\y}\cdot\frac{\Vert \z_0-\z^*\Vert}{\epsilon}\right)\right).
\end{equation}
We now compute the total gradient complexity. Recall that $\beta_1=\max\{m_\x, L_{\x\y}\}$, while $\beta_2=\max\{m_\y,L_{\x\y}\}$. By (\ref{equ:iterbound1}), (\ref{equ:iterbound2}) and (\ref{equ:iterbound3}), the total gradient complexity of Algorithm~4 to reach $\Vert \z_T-\z^*\Vert \le \epsilon$ is
\begin{equation*}
\begin{aligned}
&O\left(\sqrt{\frac{\beta_1}{m_\x}}\cdot\ln\left(\frac{L^2}{m_\x m_\y}\cdot\frac{\Vert \z_0-\z^*\Vert}{\epsilon}\right)\cdot \sqrt{\frac{\beta_2}{m_\y}}\cdot\ln\left(\frac{L^2}{m_\x m_\y}\right)\cdot \sqrt{\frac{L}{\beta_1}+\frac{L}{\beta_2}}\cdot\ln^2\left(\frac{L^2}{m_\x m_\y}\right)\right)\\
=&O\left(\sqrt{\frac{L(\beta_1+\beta_2)}{m_\x m_\y}}\cdot\ln^3\left(\frac{L^2}{m_\x m_\y}\right)\ln\left(\frac{L^2}{m_\x m_\y}\cdot\frac{\Vert \z_0-\z^*\Vert}{\epsilon}\right)\right).
\end{aligned}
\end{equation*}
If $L_{\x\y}\ge \max\{m_\x, m_\y\}$, then $\beta_1=\beta_2=L_{\x\y}$, so 
$$\sqrt{\frac{L(\beta_1+\beta_2)}{m_\x m_\y}}=\sqrt{\frac{2L\cdot L_{\x\y}}{m_\x m_\y}}\le 2\sqrt{\frac{L_\x}{m_\x}+\frac{L\cdot L_{\x\y}}{m_\x m_\y}+\frac{L_\y}{m_\y}}.$$
Now consider the case where $L_{\x\y}< \max\{m_\x,m_\y\}$. Without loss of generality, assume that $m_\x\le m_\y$. Suppose that $L_{\x\y}<m_\y$, then $L=L_\x$, $\beta_2=m_\y$, while $\beta_1\le m_\y$. Hence
$$\sqrt{\frac{L(\beta_1+\beta_2)}{m_\x m_\y}}\le \sqrt{\frac{L_\x\cdot 2m_\y}{m_\x m_\y}}=\sqrt{\frac{2L_\x}{m_\x}}\le 2\sqrt{\frac{L_\x}{m_\x}+\frac{L\cdot L_{\x\y}}{m_\x m_\y}+\frac{L_\y}{m_\y}}.$$
Thus, in either case, $\sqrt{\frac{L(\beta_1+\beta_2)}{m_\x m_\y}}=O\left(\sqrt{\frac{L_\x}{m_\x}+\frac{L\cdot L_{\x\y}}{m_\x m_\y}+\frac{L_\y}{m_\y}}\right)$. We conclude that the total gradient complexity of Algorithm~4 to find a point $\z_T=[\x_T;\y_T]$ such that $\Vert\z_T-\z^*\Vert \le \epsilon$ is
\begin{equation*}
O\left(\sqrt{\frac{L_\x}{m_\x}+\frac{L\cdot L_{\x\y}}{m_\x m_\y}+\frac{L_\y}{m_\y}}\cdot\ln^3\left(\frac{L^2}{m_\x m_\y}\right)\ln\left(\frac{L^2}{m_\x m_\y}\cdot\frac{\Vert \z_0-\z^*\Vert}{\epsilon}\right)\right).
\end{equation*} 
\end{proof}

\section{Application to Constrained Problems}
In the constrained minimax optimization problem, $\x$ is constrained to a compact convex set $\mathcal{X}\subseteq \R^n$ while $\y$ is constrained to a compact convex set $\mathcal{Y}\subseteq \R^m$. For constrained minimax optimization problems, saddle points are defined as follows.

\addtocounter{definition}{4}
\begin{definition}
$(\x^*,\y^*)$ is a saddle point of $f:\mathcal{X}\times\mathcal{Y}\to\R$ if $\forall \x\in\mathcal{X}$, $\y\in\mathcal{Y}$,
$$f(\x,\y^*)\ge f(\x^*,\y^*)\ge f(\x^*,\y).$$
\end{definition}

\begin{definition}
$(\hat\x,\hat\y)$ is an $\epsilon$-saddle point of $f:\mathcal{X}\times\mathcal{Y}\to\R$ if
$$\max_{\y\in\mathcal{Y}} f(\hat\x,\y)-\min_{\x\in\mathcal{X}} f(\x,\hat\y)\le\epsilon.$$
\end{definition}
We will use $P_{\mathcal{X}}\left[\cdot\right]$ to denote the projection onto convex set $\mathcal{X}$. Assuming efficient projection oracles, our algorithms can all be easily adapted to the constrained case. In particular, for Algorithm 1, we only need to replace AGD with the constrained version; that is, set $\x_t\gets P_{\mathcal{X}}\left[\tilde\x_{t-1}-\eta\nabla g(\tilde{x}_{t-1})\right]$.

For Algorithm 3 and 4, the modified versions are presented below. The only significant change is the addition of a projected gradient descent-ascent step in line 5-6 of Algorithm 3 and line 5-6 and 9-10 of Algorithm 4.

\subsection{Algorithmic Modifications}

\begin{algorithm}
\renewcommand{\thealgorithm}{III}
	\caption{AGD($g$, $\x_0$, $T$) with Projections~\cite[(2.2.63)]{nesterov2013introductory}}
	\begin{algorithmic}[1]
		\Require Initial point $\x_0$, smoothness constant $l$, strongly-convex modulus $m$, number of iterations $T$
		\State $\eta\gets 1/l$, $\kappa\gets l/m$, $\theta\gets (\sqrt{\kappa}-1)/(\sqrt{\kappa}+1)$
		\State 	$\x_1\gets P_{\mathcal{X}}\left[\x_0-\eta\nabla g(\x_0)\right]$, $\tilde\x_1\gets \x_1$
		\For{$t=2,\cdots,T+1$}
			\State $\x_t\gets P_{\mathcal{X}}\left[\tilde{\x}_{t-1}-\eta\nabla g(\tilde{\x}_{t-1})\right]$
			\State $\tilde{\x}_t\gets \x_t+\theta(\x_t-\x_{t-1})$
		\EndFor
	\end{algorithmic}
\end{algorithm}
For Algorithm 1, the only necessary modification is to add projection steps to the Accelerated Gradient Descent Procedure. The reason for the extra gradient step on line 2 is technical. From the original analysis~\cite[Theorem 2.2.3]{nesterov2013introductory}, it only follows that
$$\Vert \x_{T+1}-\x^*\Vert^2\le \left[\Vert\x_1-\x^*\Vert^2+\frac{2}{m}\left(f(\x_1)-f(\x^*)\right)\right]\cdot \left(1-\frac{1}{\sqrt{\kappa}}\right)^T.$$
For constrained problems, $f(\x_1)-f(\x^*)\le \frac{L}{2}\Vert\x_1-\x^*\Vert^2$ does not hold. However, with the initial projected gradient step, it can be shown that $\Vert\x_1-\x^*\Vert\le \Vert\x_0-\x^*\Vert$ and that $f(\x_1)-f(\x^*)\le \frac{L}{2}\Vert \x_0-\x^*\Vert^2$ (see Lemma~\ref{lem:projectedgd}). Thus
$$\Vert \x_{T+1}-\x^*\Vert^2 \le (\kappa+1)\Vert \x_0-\x^*\Vert^2\left(1-\frac{1}{\sqrt{\kappa}}\right)^T.$$

For Algorithm 3 and 4, the modified versions are presented below. 
\addtocounter{algorithm}{-1}
\begin{algorithm}
	\caption{APPA-ABR (for Constrained Optimization)}
	\begin{algorithmic}[1]
		\Require $g(\cdot,\cdot)$, Initial point $\z_0=[\x_0;\y_0]$, precision parameter $M_1$
		
		\State $\beta_2\gets \max\{m_\y, L_{\x\y}\}$, $M_2\gets \frac{200L^{3}}{m_\x m_\y^{2}}$
		\State $\hat\y_0\gets \y_0$ $\kappa\gets \beta_2/m_\y$, $\theta\gets \frac{2\sqrt{\kappa}-1}{2\sqrt{\kappa}+1}$, $\tau\gets \frac{1}{2\sqrt{\kappa}+4\kappa}$, $T\gets \left\lceil 8\sqrt{\kappa}\ln\left(\frac{400\kappa^2 L^2 M_1}{m_\x \sqrt{m_\x m_\y}}\right)\right\rceil$  
		\For{$t=1,\cdots,T$}
		\State $(\x'_t,\y'_t)\gets$ABR($g(\x,\y)-\beta_2\Vert\y-\hat\y_{t-1}\Vert^2,[\x_{t-1};\y_{t-1}]$, $1/M_2$, $2\beta_1$, $2\beta_2$, $3L$, $3L$)
		\State $\x_t\gets P_{\mathcal{X}}\left[\x'_t-\frac{1}{6L}\nabla_{\x}g(\x'_t,\y'_t)\right]$
		\State $\y_t\gets P_{\mathcal{Y}}\left[\y'_t+\frac{1}{6L}\left(\nabla_{\y}g(\x'_t,\y'_t)-2\beta_2(\y'_t-\hat\y_{t-1})\right)\right]$
		\State $\hat\y_t\gets \y_t+\theta(\y_t-\y_{t-1})+\tau(\y_t-\hat\y_{t-1})$
		\EndFor
	\end{algorithmic}
\end{algorithm}

\begin{algorithm}
	\caption{Proximal Best Response (for Constrained Optimization)}
	\begin{algorithmic}[1]
		\Require Initial point $\z_0=[\x_0;\y_0]$
		\State $\beta_1\gets \max\{m_\x, L_{\x\y}\}$, $M_1\gets \frac{120L^{3.5}}{m_\x^{2} m_\y^{1.5}}$
		\State $\hat\x_0\gets \x_0$, $\kappa\gets \beta_1/m_\x$, $\theta\gets \frac{2\sqrt{\kappa}-1}{2\sqrt{\kappa}+1}$, $\tau\gets \frac{1}{2\sqrt{\kappa}+4\kappa}$
		\For{$t=1,\cdots,T$}
		\State $(\x'_t,\y'_t)\gets$ APPA-ABR($f(\x,\y)+\beta_1\Vert\x-\hat\x_{t-1}\Vert^2$, $[\x_{t-1},\y_{t-1}]$, $M_1$)
		\State $\x_t\gets P_{\mathcal{X}}\left[\x'_t-\frac{1}{6L}\left(\nabla_{\x}f(\x'_t,\y'_t)+2\beta_1(\x'_t-\hat\x_{t-1})\right)\right]$
		\State $\y_t\gets P_{\mathcal{Y}}\left[\y'_t+\frac{1}{6L}\nabla_\y f(\x'_t,\y'_t)\right]$
		\State $\hat\x_t\gets \x_t+\theta(\x_t-\x_{t-1})+\tau(\x_t-\hat\x_{t-1})$
		\EndFor
		\State $ \hat\x \gets P_{\mathcal{X}}\left[\x_T-\frac{1}{2L}\nabla_{\x}f(\x_T,\y_T)\right]$
		\State $\hat\y \gets P_{\mathcal{Y}}\left[\y_T+\frac{1}{2L}\nabla_\y f(\x_T,\y_T)\right]$
	\end{algorithmic}
\end{algorithm}

The most significant change is the addition of a projected gradient descent-ascent step in line 5-6 of Algorithm 3 and line 5-6 and 9-10 of Algorithm 4. The reason for this modification is very similar to that of the initial projected gradient descent step for AGD. For unconstrained problems, a small distance to the saddle point implies a small duality gap (Fact~\ref{fact:dualitygap}); however this may not be true for constrained problems, since the saddle point may no longer be a stationary point. This is also true for minimization: if $\x^*=\argmin_{\x\in\mathcal{X}}g(\x)$ where $g(\x)$ is a $L$-smooth function $g(\x)-g(\x^*)\le \frac{L}{2}\Vert\x-\x^*\Vert^2$ may not hold.

Fortunately, there is a simple fix to this problem. By applying projected gradient descent-ascent once, we can assure that a small distance implies small duality gap. This is specified by the following lemma, which is the key reason why our result can be adapted to the constrained problem.
\begin{lemma}
\label{lem:projectedgda}
Suppose that $f\in\mathcal{F}(m_\x,m_\y,L_\x,L_{\x\y},L_\y)$, $(\x^*,\y^*)$ is a saddle point of $f$, $\z_0=(\x_0,\y_0)$ satisfies $\Vert \z_0-\z^*\Vert \le \epsilon$. Let $\hat\z=(\hat\x,\hat\y)$ be the result of one projected GDA update, i.e.
\begin{align*}
    \hat\x &\gets P_{\mathcal{X}}\left[\x_0-\frac{1}{2L}\nabla_{\x}f(\x_0,\y_0)\right],\\
    \hat\y &\gets P_{\mathcal{Y}}\left[\y_0+\frac{1}{2L}\nabla_\y f(\x_0,\y_0)\right].
\end{align*}
Then $\Vert\hat\z-\z^*\Vert\le\epsilon$, and
\begin{equation*}
    \max_{\y\in\mathcal{Y}}f(\hat\x,\y)-\min_{\x\in\mathcal{X}}f(\x,\hat\y)\le 2\left(1+\frac{L_{\x\y}^2}{\min\{m_\x,m_\y\}^2}\right)L\epsilon^2.
\end{equation*}
\end{lemma}
The proof of Lemma~\ref{lem:projectedgda} is deferred to Sec.~\ref{sec:projectedgda}. 

Because we would use Lemma~\ref{lem:projectedgda} to replace (\ref{equ:deltat}) in the analysis of Algorithm 3 and 4, we would need to accordingly increase $M_1$ to $\frac{120L^{3.5}}{m_\x^2m_\y^{1.5}}$ and $M_2$ to $\frac{200L^3}{m_\x m_\y^2}$. Apart from this, another minor change in Algorithm 3 is that it would terminate after a fixed number of iterations instead of based on a termination criterion. The number of iterations is chosen such that
$\Vert \x_T-\x^*\Vert+\Vert\y_T-\y^*\Vert\le \frac{1}{M_1}\left[\Vert\x_0-\x^*\Vert+\Vert\y_0-\y^*\Vert\right]$
is guaranteed.

\subsection{Modification of Analysis}
We now claim that after modifications to the algorithms, Theorem~\ref{thm:phr} holds for constrained cases.
\begin{manualtheorem}{3}
(Modified) Assume that $f\in\mathcal{F}(m_\x,m_\y,L_\x,L_{\x\y},L_\y)$. In Algorithm~4, the gradient complexity to find an $\epsilon$-saddle point
\begin{equation*}
O\left(\sqrt{\frac{L_\x}{m_\x}+\frac{L\cdot L_{\x\y}}{m_\x m_\y}+\frac{L_\y}{m_\y}}\cdot\ln^3\left(\frac{L^2}{m_\x m_\y}\right)\ln\left(\frac{L^2}{m_\x m_\y}\cdot\frac{L\Vert\z_0-\z^*\Vert^2}{\epsilon}\right)\right).
\end{equation*}
\end{manualtheorem}
The proof of this theorem is, for the most part, the same as the unconstrained version. Hence, we only need to point out parts of the original proof that need to be modified for the constrained case.

To start with, Theorem~\ref{thm:br} holds in the constrained case. The proof of Theorem~\ref{thm:br} only relies on the analysis of AGD and the Lipschitz properties in Fact~\ref{prop:facts}, and both still hold for constrained problems. (See \cite[Lemma B.2]{lin2020near} for the proof of Fact~\ref{prop:facts} in constrained problems.)

As for Theorem~\ref{thm:appa}, the key modification is about (\ref{equ:deltat}). As argued above, (\ref{equ:deltat}) uses the property $g(\x)-g(\x^*)\le \frac{L}{2}\Vert\x-\x^*\Vert^2$, which does not hold in constrained problems, since the optimum may not be a stationary point. Here, we would use Lemma~\ref{lem:projectedgda} to derive a similar bound to replace (\ref{equ:deltat}). Note that originally (\ref{equ:deltat}) is only used to derive $\delta_t\le  \frac{\hat L+2\beta}{2}\epsilon_t^2.$ Using Lemma~\ref{lem:projectedgda}, we can replace this with
\begin{align*}
    \delta_t &\le \max_{\y\in\mathcal{Y}}\left\{f(\x_t,\y)+\beta\Vert\x_t-\hat\x_{t-1}\Vert^2\right\}-\min_{\x\in\mathcal{X}}\left\{f(\x,\y_t)+\beta\Vert\x-\hat\x_{t-1}\Vert^2\right\}\\
    &\le 2\left(1+\frac{L_{\x\y}^2}{m_\x m_\y}\right)L\epsilon_t^2.
\end{align*}
Accordingly, we can change $C_0$ to $44\kappa\sqrt{\kappa}\cdot 2L\left(1+\frac{L_{\x\y}^2}{m_\x m_\y}\right)C_1^2$, and the assumption on $M$ to $M\ge 20\kappa \sqrt{\frac{4L}{m_\x}\left(1+\frac{L_{\x\y}^2}{m_\x m_\y}\right)}\left(1+\frac{L}{m_\y}\right)$. Then Theorem~\ref{thm:appa} would hold for the constrained case as well.

Finally, as for Theorem 3, we need to re-verify that $M_1$ and $M_2$ satisfy the new assumptions of $M$ in order to apply Theorem 2. Observe that
\begin{align*}
&20\cdot \frac{\beta_2}{m_\y}\cdot\sqrt{\frac{4(L+2\beta_1)}{m_\y}\cdot\left(1+\frac{L_{\x\y}^2}{2\beta_1 \cdot m_\y}\right)}\cdot\left(1+\frac{L}{m_\x}\right)\\
\le&  20\cdot \frac{L}{m_\y}\cdot\sqrt{\frac{18L^3}{m_\x m_\y^2}}\cdot\frac{2L}{m_\x}\le \frac{200L^3}{m_\x m_\y^2}=M_2,
\end{align*}
and that
\begin{align*}
20\cdot \frac{\beta_1}{m_\x}\cdot\sqrt{\frac{4L}{m_\x}\cdot\frac{2L_{\x\y}^2}{m_\x m_\y}}\cdot\frac{2L}{m_\y} \le \frac{80\sqrt{2}L^{3.5}}{m_\x^2 m_\y^{1.5}}\le M_1.
\end{align*}
It follows that the number of iterations needed to find $\Vert\z_T-\z^*\Vert\le \epsilon$ is
\begin{equation*}
O\left(\sqrt{\frac{L_\x}{m_\x}+\frac{L\cdot L_{\x\y}}{m_\x m_\y}+\frac{L_\y}{m_\y}}\cdot\ln^3\left(\frac{L^2}{m_\x m_\y}\right)\ln\left(\frac{L^2}{m_\x m_\y}\cdot\frac{\Vert\z_0-\z^*\Vert}{\epsilon}\right)\right).
\end{equation*}
It follows from Lemma~\ref{lem:projectedgda} that the duality gap of $(\hat\x,\hat\y)$ is at most
\begin{align*}
\max_{\y\in\mathcal{Y}} f(\hat\x,\y) - \min_{\x\in\mathcal{X}} f(\x,\hat\y) \le 2\left(1+\frac{L_{\x\y}^2}{\min\{m_\x,m_\y\}^2}\right)L\epsilon^2.
\end{align*}
Resetting $\epsilon$ to $\sqrt{\frac{\epsilon\min\{m_\x,m_\y\}^2}{4L^3}}$ proves the theorem.

\subsection{Properties of Projected Gradient}
\label{sec:projectedgda}

\begin{lemma}
\label{lem:projectedgd}
If $g:\mathcal{X}\to\R$ is $L$-smooth, $\x^*=\argmin_{\x\in\mathcal{X}}g(\x)$, $\hat\x=P_{\mathcal{X}}\left[\x_0-\frac{1}{L}\nabla g(\x_0)\right]$, then $\Vert \hat\x-\x^*\Vert\le \Vert \x_0-\x^*\Vert$, and $g(\hat\x)-g(\x^*)\le \frac{L}{2}\Vert\x_0-\x^*\Vert^2$.
\end{lemma}
\begin{proof}
By Corollary 2.2.1~\cite{nesterov2013introductory}, $(\x_0-\hat\x)^T(\x_0-\x^*)\ge \frac{1}{2}\Vert\hat\x-\x_0\Vert^2$. Therefore
\begin{align*}
\Vert \hat\x-\x^*\Vert^2 &= \Vert (\x_0-\x^*)+(\hat\x-\x_0)\Vert^2\\
&= \Vert\x_0-\x^*\Vert^2+2(\x_0-\x^*)^T(\hat\x-\x_0)+\Vert\hat\x-\x_0\Vert^2\\
&\le \Vert \x_0-\x^*\Vert^2.
\end{align*}
Meanwhile, note that $\hat\x = \argmin_{\x\in\mathcal{X}}\left\{\nabla g(\x_0)^T\x+\frac{L}{2}\Vert\x-\x_0\Vert^2\right\}$. By the optimality condition and the $L$-strong convexity of $\nabla g(\x_0)^T\x+\frac{L}{2}\Vert\x-\x_0\Vert^2$, we have
\begin{align*}
    \nabla g(\x_0)^T\hat\x + \frac{L}{2}\Vert \hat\x-\x_0\Vert^2 + \frac{L}{2}\Vert \x_1-\x^*\Vert^2 \le \nabla g(\x_0)^T\x^*+\frac{L}{2}\Vert \x^*-\x_0\Vert^2.
\end{align*}
Thus
\begin{equation*}
    \nabla g(\x_0)^T(\hat\x-\x^*) \le \frac{L}{2}\left[\Vert \x^*-\x_0\Vert^2 - \Vert \hat\x-\x_0\Vert^2 - \Vert \hat\x-\x^*\Vert^2\right].
\end{equation*}
It follows that
\begin{align*}
g(\hat\x)-g(\x^*)&\le \nabla g(\hat\x)^T(\hat\x-\x^*)\\
&= \nabla g(\x_0)^T(\hat\x-\x^*) + (\nabla g(\hat\x) - \nabla g(\x_0))^T (\hat\x-\x^*) \\
&\le \frac{L}{2}\Vert \x^*-\x_0\Vert^2 \underbrace{- \frac{L}{2}\Vert \hat\x-\x_0\Vert^2 - \frac{L}{2}\Vert \hat\x-\x^*\Vert^2 + L\Vert\hat\x-\x_0\Vert\cdot \Vert\hat\x-\x^*\Vert}_{\le 0}\\
&\le \frac{L}{2}\Vert \x^*-\x_0\Vert^2.
\end{align*}
\end{proof}

We then prove Lemma~\ref{lem:projectedgda}.
\begin{proof}[Proof of Lemma~\ref{lem:projectedgda}]
This can be seen as a special case of Proposition 2.2~\cite{nemirovski2004prox}. Define the gradient descent-ascent field to be $F(\z):=\left[\begin{matrix}
\nabla_\x f(\x,\y)\\
-\nabla_\y f(\x,\y)
\end{matrix}\right]$. Note that the $\hat\z$ can also be written as
\begin{equation*}
    \hat\z=\argmin_{\z\in\mathcal{X}\times\mathcal{Y}}\left\{L\Vert\z-\z_0\Vert^2+F(\z_0)^T\z\right\}.
\end{equation*}
Now, define $\z'=(\x',\y')$ to be
\begin{align*}
    \x' &\gets P_{\mathcal{X}}\left[\x_0-\frac{1}{2L}\nabla_{\x}f(\hat\x,\hat\y)\right],\\
    \y' &\gets P_{\mathcal{Y}}\left[\y_0+\frac{1}{2L}\nabla_\y f(\hat\x,\hat\y)\right].
\end{align*}
In other words, $\z'=\argmin_{\z\in\mathcal{X}\times\mathcal{Y}}\left\{L\Vert\z-\z_0\Vert^2+F(\hat\z)^T\z\right\}.$ By the optimality condition and $2L$-strong convexity of $L\Vert\z-\z_0\Vert^2+F(\hat\z)^T\z$, for any $\z\in\mathcal{X}\times \mathcal{Y}$,
\begin{align*}
    L\Vert\z'-\z_0\Vert^2+F(\hat\z)^T\z' + L\Vert \z'-\z\Vert^2 \le L\Vert\z-\z_0\Vert^2+F(\hat\z)^T\z.
\end{align*}
Similarly, by optimality of $\hat\z$,
\begin{align*}
    L\Vert \hat\z-\z_0\Vert^2+F(\z_0)^T\hat\z + L\Vert \z'-\hat\z\Vert^2
    \le L\Vert \z'-\z_0\Vert^2+F(\z_0)^T\z'.
\end{align*}
Thus
\begin{align*}
F(\hat\z)^T(\hat\z-\z) &= F(\hat\z)^T(\z'-\z) + F(\hat\z)^T(\hat\z-\z')\\
&= F(\hat\z)^T(\z'-\z) + F(\z_0)^T(\hat\z-\z')+(F(\hat\z)-F(\z_0))^T(\hat\z-\z')\\
&\le L\left(\Vert\z-\z_0\Vert^2-\Vert\z'-\z_0\Vert^2-\Vert\z'-\z\Vert^2\right)+(F(\hat\z)-F(\z_0))^T(\hat\z-\z')\\
&\quad + L\left(\Vert \z'-\z_0\Vert^2-\Vert\hat\z-\z_0\Vert^2-\Vert\z'-\hat\z\Vert^2\right)\\
&\le L\left(\Vert\z-\z_0\Vert^2-\Vert\z'-\z\Vert^2\right)+2L\Vert\hat\z-\z_0\Vert\cdot\Vert\hat\z-\z'\Vert-L\Vert\hat\z-\z'\Vert^2-L\Vert \hat\z-\z_0\Vert^2\\
&\le L\left(\Vert\z-\z_0\Vert^2-\Vert\z'-\z\Vert^2\right).
\end{align*}
Here we used the fact that for any $\z_1$, $\z_2$, $\Vert F(\z_1)-F(\z_2)\Vert \le 2L\Vert \z_1-\z_2\Vert$. Note that (by convexity and concavity)
\begin{align*}
F(\hat\z)^T(\hat\z-\z) &= \nabla_\x f(\hat\x,\hat\y)^T(\hat\x-\x)-\nabla_\y f(\hat\x,\hat\y)^T(\hat\y-\y)\\
&\ge \left[f(\hat\x,\hat\y)-f(\x,\hat\y)\right]+\left[f(\hat\x,\y)-f(\hat\x,\hat\y)\right]\\
&\ge f(\hat\x,\y)-f(\x,\hat\y).
\end{align*}
If we choose $\x$ and $\y$ to be $\x^*(\hat\y)$ and $\y^*(\hat\x)$, we can see that

\begin{align*}
    \max_{\y\in\mathcal{Y}}f(\hat\x,\y)-\min_{\x\in\mathcal{X}}f(\x,\hat\y)&\le L\Vert \z - \z_0\Vert^2\\
    &\le 2L\Vert \z-\z^*\Vert^2+2L\Vert\z^*-\z_0\Vert^2\\
    &\le 2L\Vert \x^*(\hat\y)-\x^*\Vert^2+2L\Vert \y^*(\hat\x)-\y^*\Vert^2+2L\Vert\z^*-\z_0\Vert^2\\
    &\le \frac{2L_{\x\y}^2}{\min\{m_\x,m_\y\}^2}\cdot L\Vert \hat\z-\z^*\Vert^2+2L\Vert\z^*-\z_0\Vert^2.
\end{align*}

By Corollary 2.2.1~\cite{nesterov2013introductory}, $\left(\x_0-\hat\x\right)^T(\x_0-\x^*)\ge \frac{1}{2}\Vert \hat\x-\x_0\Vert^2$. Therefore
\begin{align*}
\Vert \hat\x-\x^*\Vert^2 &= \Vert (\x_0-\x^*)+(\hat\x-\x_0)\Vert^2\\
&= \Vert\x_0-\x^*\Vert^2+2(\x_0-\x^*)^T(\hat\x-\x_0)+\Vert\hat\x-\x_0\Vert^2\\
&\le \Vert \x_0-\x^*\Vert^2.
\end{align*}
Similarly, $\Vert\hat\y-\y^*\Vert\le \Vert\y_0-\y^*\Vert$. Thus
$$\Vert\hat\z-\z^*\Vert^2=\Vert\hat\x-\x^*\Vert^2+\Vert\hat\y-\y^*\Vert^2\le \Vert\z_0-\z^*\Vert^2\le \epsilon^2.$$
It follows that
\begin{equation*}
    \max_{\y\in\mathcal{Y}}f(\hat\x,\y)-\min_{\x\in\mathcal{X}}f(\x,\hat\y)\le 2L\cdot \left(\frac{L_{\x\y}^2}{\min\{m_\x,m_\y\}^2}+1\right)\epsilon^2.
\end{equation*}
\end{proof}

\section{Implications of Theorem 3}
\label{append:implication}
In this section, we discuss how Theorem~\ref{thm:phr} implies improved bounds for strongly convex-concave problems and convex-concave problems via reductions established in~\cite{lin2020near}.

Let us consider minimax optimization problem $\min_{\x\in\mathcal{X}}\max_{\y\in\mathcal{Y}}f(\x,\y)$, where $f(\x,\y)$ is $m_\x$-strongly convex with respect to $\x$, concave with respect to $\y$, and $(L_\x,L_{\x\y},L_\y)$-smooth. Here, we assume that $\mathcal{X}$ and $\mathcal{Y}$ are bounded sets, with diameters $D_\x=\max_{\x,\x'\in\mathcal{X}}\Vert \x-\x'\Vert$ and $D_\y=\max_{\y,\y'\in\mathcal{Y}}\Vert\y-\y'\Vert$.

Following~\cite{lin2020near}, let us consider the function
\begin{equation*}
    f_{\epsilon,\y}(\x,\y):=f(\x,\y)-\frac{\epsilon\Vert\y-\y_0\Vert^2}{2D_\y^2}.
\end{equation*}

Recall that $(\hat\x,\hat\y)$ is an $\epsilon$-saddle point of $f$ if $\max_{\y\in\mathcal{Y}} f(\hat\x,\y)-\min_{\x\in\mathcal{X}} f(\x,\hat\y)\le\epsilon$. We now show that a $(\epsilon/2)$-saddle point of $ f_{\epsilon,\y}$ would be an $\epsilon$-saddle point of $f$. Let $\x^*(\cdot):=\argmin_{\x\in\mathcal{X}} f(\x,\cdot)$ and $\y^*(\cdot):=\argmax_{\y\in\mathcal{Y}} f(\cdot,\y)$. Obviously, for any $\x\in\mathcal{X}$, $\y\in\mathcal{Y}$,
$$f(\x,\y)-\frac{\epsilon}{2}\le f_{\epsilon,\y}(\x,\y)\le f(\x,\y).$$
Thus, if $(\hat\x,\hat\y)$ is a $(\epsilon/2)$-saddle point of $f_{\epsilon,\y}$, then
\begin{align*}
    f(\hat\x,\y^*(\hat\x)) &\le f_{\epsilon,\y}(\hat\x,\y^*(\hat\x))+\frac{\epsilon}{2}\le \max_{\y\in\mathcal{Y}} f_{\epsilon,\y}(\hat\x,\y)+\frac{\epsilon}{2},\\
    f(\x^*(\hat\y),\hat\y)&\ge f_{\epsilon,\y}(\x^*(\hat\y),\hat\y) \ge \min_{\x\in\mathcal{X}} f_{\epsilon,\y}(\x,\hat\y).
\end{align*}
It immediately follows that
\begin{equation*}
\max_{\y\in\mathcal{Y}} f(\hat\x,\y)-\min_{\x\in\mathcal{X}} f(\x,\hat\y)\le \frac{\epsilon}{2}+\max_{\y\in\mathcal{Y}} f_{\epsilon,\y}(\hat\x,\y)-\min_{\x\in\mathcal{X}} f_{\epsilon,\y}(\x,\hat\y)\le \epsilon.
\end{equation*}
Thus, to find an $\epsilon$-saddle point of $f$, we only need to find an $(\epsilon/2)$-saddle point of $f_{\epsilon,\y}$. We can now prove Corollary~\ref{cor:scc} by reducing to (the constrained version of) Theorem~\ref{thm:phr}.

Observe that $f_{\epsilon,\y}$ belongs to $\mathcal{F}(m_\x, \frac{\epsilon}{D_\y^2},L_\x,L_{\x\y},L_\y+\frac{\epsilon}{D_\y^2})$. Thus, by Theorem~\ref{thm:phr}, the gradient complexity of finding a $(\epsilon/2)$-saddle point in $f_{\epsilon,\y}$ is~
\footnote{Here it is assumed that $\epsilon$ is sufficiently small, i.e. $\epsilon\le \max\{L_{\x\y},m_\x\}D_\y^2$.}
$${O}\left(\sqrt{\frac{L_\x}{m_\x}+\left(\frac{L\cdot L_{\x\y} }{m_\x}+L_\y\right)\cdot\frac{D_\y^2}{\epsilon}}\cdot\ln^4\left(\frac{(D_\x+D_\y)^2L^2}{m_\x \epsilon}\right)\right)=\Tilde{O}\left(\sqrt{\frac{m_\x \cdot L_\y+L\cdot L_{\x\y}}{m_\x \epsilon}}\right),$$
which proves Corollary~\ref{cor:scc}.
\addtocounter{corollary}{-3}
\begin{corollary}
If $f(\x,\y)$ is $(L_\x,L_{\x\y},L_\y)$-smooth and $m_\x$-strongly convex w.r.t. $\x$, via reduction to Theorem~\ref{thm:phr}, the gradient complexity of finding an $\epsilon$-saddle point is $\tilde{O}\Bigl(\sqrt{\frac{m_\x\cdot L_\y+L\cdot L_{\x\y}}{m_\x \epsilon}}\Bigr)$.
\end{corollary}

In comparison, Lin et al.'s result in this setting is $\Tilde{O}\left(\sqrt{\frac{L^2}{m_\x \epsilon}}\right)$. Meanwhile a lower bound for this problem has been shown to be $\Omega\Bigl(\sqrt{\frac{L_{\x\y}^2}{m_\x \epsilon}}\Bigr)$~\cite{zhang2019lower}. It can be seen that when $L_{\x\y}\ll L$, our bound is a significant improvement over Lin et al.'s result, as $m_\x\cdot L_\y+L\cdot L_{\x\y}\ll L^2$. 

Similarly, if $f:\mathcal{X}\times\mathcal{Y}\to \R$ is convex with respect to $\x$, concave with respect to $\y$ and $(L_\x,L_{\x\y},L_\y)$-smooth, we can consider the function
\begin{equation*}
    f_\epsilon(\x,\y):=f(\x,\y)+\frac{\epsilon\Vert\x-\x_0\Vert^2}{4D_\x^2}-\frac{\epsilon\Vert\y-\y_0\Vert^2}{4D_\y^2}.
\end{equation*}
It can be shown that for any $\hat\x\in\mathcal{X}$,
\begin{align*}
    \max_{\y\in\mathcal{Y}}\left\{f(\hat\x,\y)+\frac{\epsilon\Vert\hat\x-\x_0\Vert^2}{4D_\x^2}-\frac{\epsilon\Vert\y-\y_0\Vert^2}{4D_\y^2}\right\}\ge \max_{\y\in\mathcal{Y}} f(\hat\x,\y)-\frac{\epsilon}{4}.
\end{align*}
Similarly, for any $\hat\y\in\mathcal{Y}$,
\begin{align*}
    \min_{\x\in\mathcal{X}} \left\{f(\x,\hat\y)+\frac{\epsilon\Vert\x-\x_0\Vert^2}{4D_\x^2}-\frac{\epsilon \Vert\hat\y-\y_0\Vert^2}{4D_\y^2}\right\}\le \min_{\x\in\mathcal{Y}} f(\x,\hat\y)+\frac{\epsilon}{4}.
\end{align*}
Therefore, if $(\hat\x,\hat\y)$ is an $(\epsilon/2)$-saddle point of $f_\epsilon$, it is an $\epsilon$-saddle point of $f$, as
\begin{equation*}
    \max_{\y\in\mathcal{Y}} f(\hat\x,\y)-\min_{\x\in\mathcal{X}}f(\x,\hat\y)\le \frac{\epsilon}{2}+\max_{\y\in\mathcal{Y}} f_\epsilon(\hat\x,\y)-\min_{\x\in\mathcal{X}}f_\epsilon(\x,\hat\y)\le \epsilon.
\end{equation*}
Observe that $f_\epsilon$ belongs to $\mathcal{F}(\frac{\epsilon}{2D_\x^2},\frac{\epsilon}{2D_\y^2},L_\x+\frac{\epsilon}{2D_\x^2},L_{\x\y},L_\y+\frac{\epsilon}{2D_\y^2})$. Thus, by Theorem~\ref{thm:phr}, the gradient complexity of finding an $(\epsilon/2)$-saddle point of $f_\epsilon$ is
\begin{equation*}
O\left(\left(\sqrt{\frac{L_\x D_\x^2+L_\y D_\y^2}{\epsilon}}+\frac{D_\x D_\y \sqrt{L\cdot L_{\x\y}}}{\epsilon}\right)\cdot \ln^4\left(\frac{L (D_\x +D_\y)^2}{\epsilon}\right)\right),
\end{equation*}
which proves Corollary~\ref{cor:scc2}.
\begin{corollary}
If $f(\x,\y)$ is $(L_\x,L_{\x\y},L_\y)$-smooth and convex-concave, via reduction to Theorem~\ref{thm:phr}, the gradient complexity to produce an $\epsilon$-saddle point is $\tilde{O}\Bigl(\sqrt{\frac{L_\x+L_\y}{\epsilon}}+\frac{\sqrt{L\cdot L_{\x\y}}}{\epsilon}\Bigr)$.
\end{corollary}
In comparison, Lin et al.'s result for this setting is $\Tilde{O}\left(\frac{L}{\epsilon}\right)$, and the classic result for ExtraGradient is $O\left(\frac{L}{\epsilon}\right)$~\cite{nemirovski2004prox}. Meanwhile, a lower bound for this setting has shown to be $\Omega\left(\sqrt{\frac{L_\x}{\epsilon}}+\frac{L_{\x\y}}{\epsilon}\right)$~\cite{ouyang2019lower}. Again, our result can be a significant improvement over Lin et al.'s result if $L_{\x\y}\ll L$, and is closer to the lower bound.

\section{Proof of Theorem 4}
\label{append:thm4proof}
The details of RHSS($k$)
can be found in Algorithm 5. We will start by proving several useful lemmas.

\noindent
{\bf Lemma~1.}
(\cite{bai2003hermitian})
{\em Define $M(\eta):=\left(\eta\matp+\mathbf{S}\right)^{-1}\left(\eta\matp-\mathbf{G}\right)\left(\eta\matp+\mathbf{G}\right)^{-1}\left(\eta\matp-\mathbf{S}\right)$. Then
	\begin{equation*}
	\rho(\mathbf{M}(\eta))\le \Vert \mathbf{M}(\eta)\Vert_2\le \max_{\lambda_i\in sp(\matp^{-1}\mathbf{G})}\left|\frac{\lambda_i-\eta}{\lambda_i+\eta}\right|<1.
	\end{equation*}
}

\addtocounter{algorithm}{0}
\begin{algorithm}[t]
	\caption{RHSS($k$) (Recursive Hermitian-skew-Hermitian Split)}
	\begin{algorithmic}
		\Require Initial point $[\x_0;\y_0]$, precision $\epsilon$, parameters $m_\x$, $m_\y$, $L_{\x\y}$
		\State $t\gets 0$, $M_1\gets \frac{192L^5}{m_\x^2 m_\y^3}$, $M_2\gets \frac{16L_{\x\y}}{m_\y}$, $\alpha\gets \frac{m_\x}{m_\y}$, $\beta\gets L_{\x\y}^{-\frac{2}{k}}m_\y^{-\frac{k-2}{k}}$, $\eta\gets L_{\x\y}^{\frac{1}{k}}m_\y^{1-\frac{1}{k}}$, $\tilde{\epsilon}\gets \frac{m_\x\epsilon}{L_{\x\y}+L_\x}$
		\Repeat
		$$\left[\begin{matrix}
		\mathbf{r}_1\\
		\mathbf{r}_2
		\end{matrix}\right]\gets \left[\begin{matrix}
		\eta\left(\alpha\iden+\beta\ma\right) & -\mb \\ \mb^T & \eta\left(\iden+\beta\mc\right)
		\end{matrix}\right]\left[\begin{matrix}
		\mathbf{x}_t\\
		\mathbf{y}_t
		\end{matrix}\right]+\left[\begin{matrix}
		-\mathbf{u}\\
		\mathbf{v}
		\end{matrix}\right].$$
		\State Call conjugate gradient to compute
		$$\left[\begin{matrix}
		\x_{t+1/2}\\
		\y_{t+1/2}
		\end{matrix}\right]\gets \text{CG}\left(\left[\begin{matrix}
		\eta\left(\alpha\iden+\beta\ma\right)+\ma & \\  & \eta\left(\iden+\beta\mc\right)+\mc
		\end{matrix}\right],\left[\begin{matrix}
		\mathbf{r}_1\\
		\mathbf{r}_2
		\end{matrix}\right],\left[\begin{matrix}
		\x_t\\
		\y_t
		\end{matrix}\right], \frac{1}{M_1}\right).$$
		
		\State Compute
		$$\left[\begin{matrix}
		\mathbf{{w}}_1\\
		\mathbf{{w}}_2
		\end{matrix}\right]\gets \left[\begin{matrix}
		\eta\alpha\iden+\eta\beta\ma-\ma & 0 \\ 0 & \eta\left(\iden+\beta\mc\right)-\mc
		\end{matrix}\right]\left[\begin{matrix}
		{\mathbf{x}}_{t+1/2}\\
		{\y}_{t+1/2}
		\end{matrix}\right]+\left[\begin{matrix}
		-\mathbf{u}\\
		\mathbf{v}
		\end{matrix}\right]$$
		\State Call RHSS($k-1$) with initial point $[\x_t;\y_t]$ and precision $1/M_2$ to solve
		$$\left[\begin{matrix}
		\x_{t+1}\\
		\y_{t+1}
		\end{matrix}\right]\gets \left[\begin{matrix}
		\eta\left(\alpha\iden+\beta\ma\right) & \mb\\  -\mb^T& \eta\left(\iden+\beta\mc\right)
		\end{matrix}\right]^{-1}\left[\begin{matrix}
		\mathbf{w}_1\\
		\mathbf{w}_2
		\end{matrix}\right].$$
		\State $t\gets t+1$
		\Until{$\Vert \jacobian\z_{t}-\bb\Vert\le \tilde{\epsilon}\Vert\jacobian\z_0-\bb\Vert$}
	\end{algorithmic}
\end{algorithm}

\begin{algorithm}
	\caption{The Conjugate Gradient Algorithm: CG($\ma$,$\mathbf{b}$,$\x_0$,$\epsilon$)~\cite{allaire2008numerical}
	}
	\begin{algorithmic}
		\State $\mathbf{r}_0\gets \mathbf{b}-\ma\x_0$, $\mathbf{p}_0\gets \mathbf{r}_0$, $k\gets 0$
		\Repeat
		    \State $\alpha_k\gets \frac{\mathbf{r}_k^T\mathbf{r}_k}{\mathbf{p}_k^T\ma\mathbf{p}_k}$
		    \State $\x_{k+1}\gets \x_k+\alpha_k \mathbf{p}_k$
		    \State $\mathbf{r}_{k+1}\gets \mathbf{r}_k-\alpha_k\ma\mathbf{p}_k$
		    \State $\beta_k\gets \frac{\mathbf{r}_{k+1}^T\mathbf{r}_{k+1}}{\mathbf{r}_k^T\mathbf{r}_k}$
		    \State $\mathbf{p}_{k+1}\gets \mathbf{r}_{k+1}+\beta_k\mathbf{p}_k$
		    \State $k\gets k+1$
		\Until{{$\Vert \mathbf{r}_k\Vert\le \epsilon\Vert \bb-\ma\x_0\Vert$}}
		\State Return $\x$
	\end{algorithmic}
\end{algorithm}

\begin{proof}[Proof of Lemma 1]
We provide a proof for completeness.
First, observe that
\begin{align*}
\mathbf{M}(\eta)&=(\eta \matp + \mathbf{S})^{-1}(\eta \matp - \mathbf{G})(\eta \matp + \mg)^{-1}(\eta \matp - \mathbf{S})\\
&= \matp^{-\frac{1}{2}}(\eta \iden + \matp^{-\frac{1}{2}}\mathbf{S}\matp^{-\frac{1}{2}})^{-1}(\eta \iden - \matp^{-\frac{1}{2}}\mathbf{G}\matp^{-\frac{1}{2}})(\eta \iden + \matp^{-\frac{1}{2}}\mg\matp^{-\frac{1}{2}})^{-1}(\eta \iden - \matp^{-\frac{1}{2}}\mathbf{S}\matp^{-\frac{1}{2}})\matp^{\frac{1}{2}}.
\end{align*}
Let $\hat{\mathbf{G}}:=\matp^{-\frac{1}{2}}\mathbf{G}\matp^{-\frac{1}{2}}$, $\hat{\mathbf{S}}:=\matp^{-\frac{1}{2}}\mathbf{S}\matp^{-\frac{1}{2}}$. Then $\mathbf{M}(\eta)$ is similar to
$$(\eta\iden+\hat{\mathbf{S}})^{-1}(\eta\iden-\hat{\mathbf{G}})(\eta\iden+\hat{\mathbf{G}})^{-1}(\eta\iden-\hat{\mathbf{S}}),$$
which is then similar to
$$(\eta\iden-\hat{\mathbf{G}})(\eta\iden+\hat{\mathbf{G}})^{-1}(\eta\iden-\hat{\mathbf{S}})(\eta\iden+\hat{\mathbf{S}})^{-1}.$$
The key observation is that $(\eta\iden-\hat{\mathbf{S}})(\eta\iden+\hat{\mathbf{S}})^{-1}$ is orthogonal, since
\begin{align*}
&\left((\eta\iden+\hat{\mathbf{S}})^{-1}\right)^T(\eta\iden-\hat{\mathbf{S}})^T	(\eta\iden-\hat{\mathbf{S}})(\eta\iden+\hat{\mathbf{S}})^{-1}\\
=&(\eta\iden-\hat{\mathbf{S}})^{-1}(\eta\iden+\hat{\mathbf{S}})	(\eta\iden-\hat{\mathbf{S}})(\eta\iden+\hat{\mathbf{S}})^{-1}\\
=&(\eta\iden-\hat{\mathbf{S}})^{-1}(\eta\iden-\hat{\mathbf{S}})	(\eta\iden+\hat{\mathbf{S}})(\eta\iden+\hat{\mathbf{S}})^{-1}=\iden.
\end{align*}
Therefore
\begin{align*}
\rho(\mathbf{M}(\eta))&\le \Vert(\eta\iden-\hat{\mathbf{G}})(\eta\iden+\hat{\mathbf{G}})^{-1}(\eta\iden-\hat{\mathbf{S}})(\eta\iden+\hat{\mathbf{S}})^{-1}\Vert_2\\
&\le \Vert(\eta\iden-\hat{\mathbf{G}})(\eta\iden+\hat{\mathbf{G}})^{-1}\Vert_2\cdot \Vert(\eta\iden-\hat{\mathbf{S}})(\eta\iden+\hat{\mathbf{S}})^{-1}\Vert_2\\
&=\Vert(\eta\iden-\hat{\mathbf{G}})(\eta\iden+\hat{\mathbf{G}})^{-1}\Vert_2\\
&=\max_{\lambda_i\in sp(\hat{\mathbf{G}})}\left|\frac{\lambda_i-\eta}{\lambda_i+\eta}\right|=\max_{\lambda_i\in sp(\matp^{-1}\mathbf{G})}\left|\frac{\lambda_i-\eta}{\lambda_i+\eta}\right|.
\end{align*}
\end{proof}

We now proceed to state some useful lemmas for the proof of Theorem~\ref{thm:rhss}.

\begin{lemma}
\label{lem:facts}
The following statements about the eigenvalues and singular values of matrices hold:
\begin{enumerate}
\item The singular values of $\jacobian$ fall in $[m_\x, L_{\x\y}+L_\x]$;
\item The condition number of $\eta\matp+\mathbf{G}$ is at most $\frac{3L_\x}{m_\x}\left(\frac{m_\y}{L_{\x\y}}\right)^{\frac{1}{k}}$;
\item The condition number of $\eta\matp+\mathbf{G}$ is at most $L_\x/m_\x$.
\item The eigenvalues of $\eta(\alpha\iden+\beta\ma)$ fall in $[\eta\alpha, 2\eta\beta L_\x]$. The eigenvalues of $\eta(\iden+\beta\mc)$ fall in $[\eta, 2\eta\beta L_\x]$.
\end{enumerate}
\end{lemma}
\begin{proof}[Proof of Lemma \ref{lem:facts}]
1. Consider an arbitrary $\x\in \R^{n+m}$ with $\Vert \x\Vert_2=1$. Construct a set of orthonormal vectors $\{\x_1,\cdots,\x_{n+m}\}$ with $\x_1=\x$. Then
	\begin{align*}
	\x^T\jacobian^T\jacobian\x &= \sum_{i=1}^{n+m}\x^T\jacobian^T\x_i\x_i^T\jacobian\x=\sum_{i=1}^{n+m}\left(\x^T\jacobian^T\x_i\right)^2\ge \left(\x^T\jacobian^T\x\right)^2.
	\end{align*} 
	Since $\jacobian=\mathbf{G}+\mathbf{S}=\left[\begin{matrix}
	\ma & 0 \\ 0 & \mc
	\end{matrix}\right]+\left[\begin{matrix}
	0 & \mb \\ -\mb^T & 0
	\end{matrix}\right]$, where $\mathbf{S}$ is skew-symmetric, $\x^T\jacobian^T\x=\x^T\mathbf{G}\x\ge m_{\x}$. Thus
	\begin{align*}
	\sigma_{min}(\jacobian)=\sqrt{\lambda_{\min}\left(\jacobian^T\jacobian\right)}\ge m_{\x}.
	\end{align*}
	Meanwhile,
	\begin{align*}
	\lambda_{\max}\left(\jacobian^T\jacobian\right)\le \Vert \mathbf{G}\Vert_2+\Vert \mathbf{S}\Vert_2\le L_{\x\y}+L_\x.
	\end{align*}

2. Note that
	\begin{equation*}
	\eta\matp+\mathbf{G}=\left[\begin{matrix}
	\eta(\alpha\iden+\beta\ma)+\ma & \\ & \eta(\iden+\beta\mc)+\mc
	\end{matrix}\right].
	\end{equation*}
	Thus
	\begin{equation*}
	\Vert \eta\matp+\mathbf{G}\Vert_2 \le \max\{\eta\left(\alpha+\beta L_\x\right)+L_\x,\eta\left(1+\beta L_\x\right)+L_\x\}=\eta\left(1+\beta L_\x\right)+L_\x.
	\end{equation*}
	On the other hand
	\begin{equation*}
	\lambda_{\min}\left(\eta\matp+\mathbf{G}\right)\ge \min\{\eta\alpha+\eta\beta m_\x+m_\x, \eta+\eta\beta m_\y+m_\y\}=\eta\alpha+\eta\beta m_\x+m_\x.
	\end{equation*}
	Thus the condition number of $\eta\matp+\mathbf{G}$ is at most
	\begin{equation*}
	\frac{\eta\left(1+\beta L_\x\right)+L_\x}{\eta\alpha+\eta\beta m_\x+m_\x}\le \frac{L_\x}{\eta\alpha}+\frac{1+\beta L_\x}{\alpha+\beta m_\x}\le \frac{L_\x}{\eta\alpha}+\frac{2\beta L_\x}{\alpha}=\frac{L_\x}{m_\x}\left(\frac{m_\y}{L_{\x\y}}\right)^{\frac{1}{k}}+\frac{2L_\x}{m_\x}\left(\frac{m_\y}{L_{\x\y}}\right)^{\frac{2}{k}}\le \frac{3L_\x}{m_\x}\left(\frac{m_\y}{L_{\x\y}}\right)^{\frac{1}{k}}.
	\end{equation*}
3. On the other hand,
	\begin{equation*}
	\frac{\eta\left(1+\beta L_\x\right)+L_\x}{\eta\alpha+\eta\beta m_\x+m_\x}=\frac{\eta+\eta\beta L_\x + L_\x}{\eta\alpha+\eta\beta m_\x+m_\x}\le \max\left\{\frac{1}{\alpha},\frac{L_\x}{m_\x}\right\}=\frac{L_\x}{m_\x}.
	\end{equation*}
	
4. Finally let us consider matrices $\eta(\alpha\iden+\beta\ma)$ and $\eta(\iden+\beta\mc)$. Obviously
	$$\eta(\alpha\iden+\beta\ma)\succcurlyeq \eta\alpha\iden,\quad \eta(\iden+\beta\mc)\succcurlyeq \eta\iden.$$
	Meanwhile
	\begin{align*}
	\Vert \eta(\alpha\iden+\beta\ma)\Vert &\le \eta\cdot\left(\alpha+\beta L_\x\right)\\
	&\le \eta\left(1+\beta L_\x\right)\tag{$\alpha<1$}\\
	&\le 2\eta \beta L_\x. \tag{$\beta L_\x>1$}
	\end{align*}
	Similarly $\Vert \eta(\iden+\beta\mc)\Vert \le 2\eta \beta L_\x$.
\end{proof}

\begin{lemma}
\label{lem:rhsscontraction2}
With our choice of $\eta$, $\alpha$ and $\beta$,
$$\rho(\mathbf{M}(\eta))\le \Vert \mathbf{M}(\eta)\Vert_2\le 1-\frac{1}{2}\left(\frac{m_\y}{L_{\x\y}}\right)^{\frac{1}{k}}.$$
\end{lemma}
\begin{proof}[Proof of Lemma~\ref{lem:rhsscontraction2}]
	By Lemma~1,
	\begin{equation*}
	\rho(\mathbf{M}(\eta))\le \Vert \mathbf{M}(\eta)\Vert_2\le \max_{\lambda_i\in sp(\matp^{-1}\mathbf{G})}\left|\frac{\lambda_i-\eta}{\lambda_i+\eta}\right|.
	\end{equation*}
	Observe that
	\begin{equation*}
	\matp^{-1}\mathbf{G}=\left[\begin{matrix}
	(\alpha\iden+\beta\ma)^{-1}\ma & \\ & (\iden+\beta\mc)^{-1}\mc
	\end{matrix}\right].
	\end{equation*}
	The eigenvalues of $(\alpha\iden+\beta\ma)^{-1}\ma$ are contained in
	\begin{equation*}
	\left[\frac{m_\x}{\alpha+\beta m_\x},\frac{L_\x}{\alpha+\beta L_\x}\right]\subseteq \left[\frac{m_\y}{2},\frac{1}{\beta}\right]. \tag{$\beta m_\x \le \alpha$}
	\end{equation*}
	Similarly the eigenvalues of $(\iden+\beta\mc)^{-1}\mc$ are contained in
	\begin{equation*}
	\left[\frac{m_\y}{1+\beta m_\y},\frac{L_\x}{1+\beta L_\x}\right]\subseteq \left[\frac{m_\y}{2},\frac{1}{\beta}\right]. \tag{$\beta m_\y \le 1$}
	\end{equation*}
	Recall that $\eta=L_{\x\y}^{1/k}m_\y^{1-1/k}=\sqrt{m_\y/\beta}$. As a result,
	\begin{align*}
	\max_{\lambda_i\in sp(\matp^{-1}\mathbf{G})}\left|\frac{\lambda_i-\eta}{\lambda_i+\eta}\right|\le \max\left\{\frac{\frac{1}{\beta}-\sqrt{\frac{m_\y}{\beta}}}{\frac{1}{\beta}+\sqrt{\frac{m_\y}{\beta}}},\frac{\sqrt{\frac{m_\y}{\beta}}-\frac{m_\y}{2}}{\sqrt{\frac{m_\y}{\beta}}+{\frac{m_\y}{2}}}\right\}\le 1-\frac{\sqrt{\beta m_\y}}{2}=1-\frac{1}{2}\left(\frac{m_\y}{L_{\x\y}}\right)^{\frac{1}{k}}.
	\end{align*}
	
\end{proof}

\begin{lemma}
	\label{lem:rhssguar}
When RHSS($k$) terminates  $\Vert \z_t-\z^*\Vert\le \epsilon\Vert \z_0-\z^*\Vert$.
\end{lemma}
\begin{proof}[Proof of Lemma \ref{lem:rhssguar}]
$$\sigma_{\min}(\jacobian)\Vert \z_t-\z^*\Vert\le \Vert\jacobian\z_t-\bb\Vert\le \tilde{\epsilon}\Vert \jacobian\z_0-\bb\Vert\le \sigma_{\max}(\jacobian)\Vert \z_0-\z^*\Vert.$$
We know that $\sigma_{\min}(\jacobian)\ge m_\x$ and that $\sigma_{\max}(\jacobian)\le L_\x+L_{\x\y}$. Thus
\begin{equation*}
\Vert \z_t-\z^*\Vert\le \frac{\tilde\epsilon \cdot (L_\x+L_{\x\y})}{m_\x}\Vert \z_0-\z^*\Vert = \epsilon\Vert \z_0-\z^*\Vert.
\end{equation*}
\end{proof}

\begin{lemma}[Proposition 9.5.1, \cite{allaire2008numerical}]
	\label{lem:cg}
	CG($\ma,\bb,\x_0,\epsilon$) returns (i.e. satisfies $\Vert \ma\x_T-\bb\Vert\le \epsilon\Vert \ma\x_0-\bb\Vert$) in at most $\left\lceil \sqrt{\kappa}\ln\left(\frac{2\sqrt{\kappa}}{\epsilon}\right)\right\rceil$ iterations.
\end{lemma}

\begin{lemma}
\label{lem:inexactcontraction}
In RHSS($k$), 
\begin{equation}
\label{equ:inexactcontraction}
\Vert \z_{t+1}-\z^*\Vert \le \left(1-\frac{1}{4}\left(\frac{m_\y}{L_{\x\y}}\right)^{\frac{1}{k}}\right)\Vert \z_t-\z^*\Vert.
\end{equation}
\end{lemma}
\begin{proof}[Proof of Lemma~\ref{lem:inexactcontraction}]
	Let us define
	\begin{align*}
	\tilde{\z}_{t+1/2}=\left[\begin{matrix}
	\tilde{\x}_{t+1/2}\\ \tilde{\y}_{t+1/2}
	\end{matrix}\right]=\left[\begin{matrix}
	\eta\left(\alpha\iden+\beta\ma\right)+\ma & \\  & \eta\left(\iden+\beta\mc\right)+\mc
	\end{matrix}\right]^{-1}\left[\begin{matrix}
	\mathbf{r}_1\\ \mathbf{r}_2
	\end{matrix}\right].
	\end{align*}
	Since $\Vert (\eta\matp+\mathbf{G})(\z_{t+1/2}-\tilde\z_{t+1/2}\Vert \le \frac{1}{M_1}\Vert (\eta\matp+\mathbf{G})(\z_{t}-\tilde\z_{t+1/2}\Vert$,
	\begin{equation}
	\label{equ:initialdist1}
	\begin{aligned}
	\Vert \z_{t+1/2}-\tilde\z_{t+1/2}\Vert &\le \frac{\Vert (\eta\matp+\mathbf{G})(\z_{t+1/2}-\tilde\z_{t+1/2}\Vert}{\lambda_{\min}( \eta\matp+\mathbf{G})}\le
	\frac{\Vert (\eta\matp+\mathbf{G})(\z_{t}-\tilde\z_{t+1/2}\Vert}{M_1 \lambda_{\min}( \eta\matp+\mathbf{G})} \\
	&\le \frac{\lambda_{\max}(\eta\matp+\mathbf{G})}{M_1 \lambda_{\min}( \eta\matp+\mathbf{G})}\Vert \z_{t}-\tilde\z_{t+1/2}\Vert\\
	&\le \frac{L_\x}{M_1 m_\x}\Vert \z_{t}-\tilde\z_{t+1/2}\Vert=\frac{m_\x m_\y^3}{192L^4}\Vert \z_t-\tilde{\z}_{t+1/2}\Vert.
	\end{aligned}
	\end{equation}
	Because $\tilde{\z}_{t+1/2}-\z^*=\left(\eta\matp+\mathbf{G}\right)^{-1}\left(\eta\matp-\mathbf{S}\right)(\z_t-\z^*),$
	\begin{align*}
	\Vert \tilde{\z}_{t+1/2}-\z^*\Vert &\le \Vert (\eta\matp+\mathbf{G})^{-1}\Vert_2 \Vert \eta\matp-\mathbf{S}\Vert_2 \Vert \z_t-\z^*\Vert\\
	&\le \frac{1}{\eta\alpha}\cdot (L_{\x\y}+\eta\alpha+\eta\beta L_\x)\cdot  \Vert \z_t-\z^*\Vert\\
	&\le \left(1+\frac{2L}{m_\x}\right)\Vert \z_t-\z^*\Vert.
	\end{align*}
	It follows that
	\begin{align*}
	\Vert \z_t-\tilde{\z}_{t+1/2}\Vert &\le \Vert \z_t-\z^*\Vert + \Vert \tilde{\z}_{t+1/2}-\z^*\Vert \le \left(2+\frac{2L}{m_\x }\right)\Vert \z_t-\z^*\Vert.
	\end{align*}
	By plugging this into (\ref{equ:initialdist1}), one gets
	\begin{equation}
	\Vert \z_{t+1/2}-\tilde\z_{t+1/2}\Vert\le \frac{m_\x m_\y^3}{192L^4}\cdot \left(2+\frac{2L}{m_\x }\right)\Vert \z_t-\z^*\Vert\le \frac{m_\y^3}{48L^3}\Vert \z_t-\z^*\Vert.
	\end{equation}
	Now, let us define
	\begin{align*}
	\tilde{\z}_{t+1}&:=\left(\eta\matp+\mathbf{S}\right)^{-1}\left[(\eta\matp-\mathbf{G})\tilde\z_{t+1/2}+\bb\right],\\
	\hat{\z}_{t+1}&:=\left(\eta\matp+\mathbf{S}\right)^{-1}\left[(\eta\matp-\mathbf{G})\z_{t+1/2}+\bb\right].
	\end{align*}
	First let us try to bound $\Vert \tilde{\z}_{t+1}-\hat{\z}_{t+1}\Vert$. Observe that $\hat{\z}_{t+1}-\z^*=(\eta\matp+\mathbf{S})^{-1}(\eta\matp-\mathbf{G})(\z_{t+1/2}-\z^*)$, so
	\begin{align}
	\Vert \tilde{\z}_{t+1}-\hat{\z}_{t+1}\Vert &= \Vert (\tilde{\z}_{t+1}-\z^*)-(\hat{\z}_{t+1}-\z^*)\Vert \nonumber \\
	&=\Vert \left(\eta\matp+\mathbf{G}\right)^{-1}\left(\eta\matp-\mathbf{S}\right)(\tilde\z_{t+1/2}-\z_{t+1/2})\Vert \nonumber \\
	&\le \Vert \left(\eta\matp+\mathbf{G}\right)^{-1}\left(\eta\matp-\mathbf{S}\right)\Vert_2 \cdot \Vert\tilde\z_{t+1/2}-\z_{t+1/2}\Vert \nonumber \\
	&\le \frac{3L^2}{m_\y^2}\cdot \frac{m_\y^3}{48L^3}\Vert \z_t-\z^*\Vert=\frac{m_\y}{16L}\Vert \z_t-\z^*\Vert. \label{equ:zt}
	\end{align}
	Next, by Lemma~\ref{lem:rhssguar} on RHSS($k-1$),
	\begin{align*}
	\Vert \z_{t+1}-\hat{\z}_{t+1}\Vert \le \frac{1}{M_2}\Vert \z_{t}-\hat{\z}_{t+1}\Vert\le \frac{1}{M_2}\left(\Vert \z_{t}-\z^*\Vert+\Vert \hat\z_{t+1}-\z^*\Vert\right).
	\end{align*}
	By Lemma~\ref{lem:rhsscontraction2},
	\begin{equation}
	\label{equ:zt2}
	\Vert\tilde\z_{t+1}-\z^*\Vert= \Vert \mathbf{M}(\eta)(\z_t-\z^*)\Vert_2\le \left(1-\frac{1}{2}\left(\frac{m_\y}{L_{\x\y}}\right)^{\frac{1}{k}}\right)\Vert \z_t-\z^*\Vert.
	\end{equation}
	Thus
	\begin{equation}
	\label{equ:zt3}
	\Vert \z_{t+1}-\hat\z_{t+1}\Vert \le \frac{1}{M_2}\left(2\Vert \z_t-\z^*\Vert+\Vert \tilde{\z}_{t+1}-\hat\z_{t+1}\Vert\right).
	\end{equation}
	Combining (\ref{equ:zt}) and (\ref{equ:zt3}), one gets
	\begin{align*}
	\Vert \z_{t+1}-\tilde{\z}_{t+1}\Vert&\le \Vert  \z_{t+1}-\hat{\z}_{t+1}\Vert + \Vert  \hat\z_{t+1}-\tilde{\z}_{t+1}\Vert\\
	&\le \frac{2}{M_2}\Vert \z_t-\z^*\Vert+\left(1+\frac{2}{M_2}\right)\Vert  \hat\z_{t+1}-\tilde{\z}_{t+1}\Vert\\
	&\le \frac{m_\y}{8L_{\x\y}}\Vert \z_t-\z^*\Vert+\frac{m_\y}{8L}\Vert \z_t-\z^*\Vert \le \frac{m_\y}{4L_{\x\y}}\Vert \z_t-\z^*\Vert.
	\end{align*}
	Combining this with (\ref{equ:zt2}), one gets
	\begin{align*}
	\Vert \z_{t+1}-\z^*\Vert &\le \Vert\tilde{\z}_{t+1}-\z^*\Vert+\Vert\tilde{\z}_{t+1}-\z_{t+1}\Vert \\
	&\le \left(1-\frac{1}{2}\left(\frac{m_\y}{L_{\x\y}}\right)^{\frac{1}{k}}\right)\Vert \z_t-\z^*\Vert + \frac{m_\y}{4L_{\x\y}}\Vert \z_t-\z^*\Vert\\
	&\le \left(1-\frac{1}{4}\left(\frac{m_\y}{L_{\x\y}}\right)^{\frac{1}{k}}\right)\Vert \z_t-\z^*\Vert.
	\end{align*}
\end{proof}

Finally, we are ready to prove Theorem~\ref{thm:rhss}.
\begin{manualtheorem}{4}
There exists constants $C_1$, $C_2$, such that the number of matrix-vector products needed to find $(\x_T,\y_T)$ such that $\Vert\z_T-\z^*\Vert\le \epsilon$ is at most
\begin{equation}
\label{equ:rhssbound2}
\sqrt{\frac{L_{\x\y}^2}{m_\x m_\y}+\left(\frac{L_\x}{m_\x}+\frac{L_\y}{m_\y}\right)\left(1+\left(\frac{L_{\x\y}}{\max\{m_\x,m_\y\}}\right)^{\frac{1}{k}}\right)}\cdot \left(C_1\ln\left(\frac{C_2 L^2}{m_\x m_\y}\right)\right)^{k+3}\ln\left(\frac{\Vert \z_0-\z^*\Vert}{\epsilon}\right).
\end{equation}
\end{manualtheorem}
\begin{proof}[Proof of Theorem~\ref{thm:rhss}]
By Lemma~\ref{lem:inexactcontraction}, when running RHSS($k$),
$$\Vert \z_{T}-\z^*\Vert \le \left(1-\frac{1}{4}\left(\frac{m_\y}{L_{\x\y}}\right)^{\frac{1}{k}}\right)^T\Vert \z_0-\z^*\Vert.$$
Thus, when $T> 4\left(\frac{L_{\x\y}}{m_\y}\right)^{1/k}\cdot \ln\left(\frac{\Vert \z_0-\z^*\Vert}{\epsilon}\right)$, one can ensure that $\Vert \z_T-\z^*\Vert \le \epsilon$. Now we can focus on the number of matrix-vector products needed per iteration, which comes in two parts: the cost of calling conjugate gradient and the cost of calling RHSS($k-1$).

\paragraph{Conjugate Gradient cost}
The matrix to be solved via conjugate gradient is $\eta\matp+\mathbf{G}$. By Lemma~\ref{lem:facts}, its condition number is upper bounded by $\frac{3L_\x}{m_\x}\left(\frac{m_\y}{L_{\x\y}}\right)^{1/k}$. By Lemma~\ref{lem:cg}, the number of matrix-vector products needed for calling CG is
\begin{equation}
\label{equ:cgcost}
\left\lceil \sqrt{\frac{3L_\x}{m_\x}\left(\frac{m_\y}{L_{\x\y}}\right)^{1/k}}\ln\left(2\sqrt{\frac{3L_\x}{m_\x}\left(\frac{m_\y}{L_{\x\y}}\right)^{1/k}}M_1\right)\right\rceil\le c_1\sqrt{\frac{L_\x}{m_\x}\left(\frac{m_\y}{L_{\x\y}}\right)^{\frac{1}{k}}}\cdot \ln\left(\frac{c_2L^2}{m_\x m_\y}\right),
\end{equation}
for some constants $c_1,c_2>0$.

\paragraph{RHSS($k-1$) cost}
By Lemma~\ref{lem:facts}, the new saddle point problem involving $\eta\matp+\mathbf{S}$ has parameters $m'_\x=\eta\alpha$, $m'_\y=\eta$, $L'_\x=L'_\y=2\eta \beta L_\x$, $L'_{\x\y}=L_{\x\y}$. It is easy to see that $m'_\y = \eta \ge m_\y$, $m'_\x = (m_\x/m_\y)m'_\y \ge m_\x$, and that $L'_\x=L'_\y\le 2L_\x$. Thus $L'=\max\{L'_\x,L'_\y,L'_{\x\y}\}\le 2L$. Assuming that Theorem~\ref{thm:rhss} holds for RHSS($k-1$), then the number of matrix-vector products needed for the new saddle point problem can be bounded by
\begin{align*}
\underbrace{\sqrt{\frac{L_{\x\y}^2}{m'_\x m'_\y}+\left(\frac{L'_\x}{m'_\x}+\frac{L'_\y}{m'_\y}\right)\left(1+\left(\frac{L_{\x\y}}{\max\{m'_\x,m'_\y\}}\right)^{1/(k-1)}\right)}}_{(a)}\cdot \left(C_1\ln\left(\frac{C_2 L'^2}{m'_\y m'_\x}\right)\right)^{k+2}\cdot \underbrace{\ln\left(\frac{4L^2 M_2}{m_\x^2}\right)}_{(b)}.
\end{align*}
Here we used Lemma~\ref{lem:rhssguar}, that when $\Vert \z_t-\z^*\Vert \le \left(\frac{m_\x}{L_\x + L_{\x\y}}\right)^2\Vert\z_0-\z^*\Vert$, RHSS($k-1$) returns. Assume that $C_1>8$. Note that
\begin{align*}
\frac{L_{\x\y}^2}{m'_\x m'_\y} &= \frac{L_{\x\y}^2}{\alpha \eta^2} = \frac{L_{\x\y}^2}{\frac{m_\x}{m_\y}\cdot L_{\x\y}^{2/k}m_\y^{2-2/k}}=\frac{L_{\x\y}^2}{m_\x m_\y}\left(\frac{m_\y}{L_{\x\y}}\right)^{\frac{2}{k}},\\
\frac{L_{\x\y}}{m'_\y}&= \frac{L_{\x\y}}{\eta} = \frac{L_{\x\y}}{L_{\x\y}^{1/k}m_\y^{1-1/k}}= \left(\frac{L_{\x\y}}{m_\y}\right)^{\frac{k-1}{k}}.
\end{align*}
Therefore
\begin{align*}
(a)&\le\sqrt{\frac{L_{\x\y}^2}{m_\x m_\y}\left(\frac{m_\y}{L_{\x\y}}\right)^{\frac{2}{k}} + \frac{2L_\x}{m_\x}\left(\frac{m_\y}{L_{\x\y}}\right)^{\frac{2}{k}}\cdot\left(1+\left(\frac{L_{\x\y}}{m_\y}\right)^{\frac{1}{k}}\right)}\\
&\le 2\sqrt{\frac{L_{\x\y}^2}{m_\x m_\y}\left(\frac{m_\y}{L_{\x\y}}\right)^{\frac{2}{k}} + \frac{L_\x}{m_\x}\left(\frac{m_\y}{L_{\x\y}}\right)^{\frac{1}{k}}},\\
\ln\left(\frac{C_2 L'^2}{m'_\y m'_\x}\right)&\le \ln\left(\frac{4C_2 L^2}{m_\x m_\y}\right)\le 2\ln\left(\frac{C_2 L^2}{m_\x m_\y}\right),\\
(b)&\le \ln\left(\frac{64L^4}{m_\x^2 m_\y^2}\right)\le 2\ln\left(\frac{8L^2}{m_\x m_\y}\right).
\end{align*}
Thus the cost of calling RHSS($k-1$) is at most
\begin{equation}
\label{equ:rhsscost}
4C_1^{k+2}\ln^{k+3}\left(\frac{C_2 L^2}{m_\x m_\y}\right)\sqrt{\frac{L_{\x\y}^2}{m_\x m_\y}\left(\frac{m_\y}{L_{\x\y}}\right)^{\frac{2}{k}} + \frac{L_\x}{m_\x}\left(\frac{m_\y}{L_{\x\y}}\right)^{\frac{1}{k}}}.
\end{equation}

In the case where $k=2$, RHSS($k-1$) is exactly Proximal Best Response (Algorithm~4). Hence, by Theorem~\ref{thm:phr}, the number of matrix-vector products needed is at most
\begin{align*}
&O\left(\sqrt{\frac{L_{\x\y}\cdot \max\{L_{\x\y},L'\}}{m'_\x m'_\y}+\frac{L'_\x }{m'_\x}+\frac{L'_\y}{m'_\y}}\cdot \ln^4\left(\frac{L'^2}{m'_\x m'_\y}\right)\ln\left(\frac{L' M_2}{m'_\x m'_\y}\right)\right)\\
=&O\left(\sqrt{\frac{L_{\x\y}}{m_\x}+\frac{L_\x \sqrt{m_\y}}{m_\x \sqrt{L_{\x\y}}}}\ln^5\left(\frac{L^2}{m_\x m_\y}\right)\right).
\end{align*}
By this, we mean there exists constants $c_3,c_4>0$ such that the number of matrix-vector products needed is
\begin{align*}
c_3\sqrt{\frac{L_{\x\y}}{m_\x}+\frac{L_\x \sqrt{m_\y}}{m_\x \sqrt{L_{\x\y}}}}\ln^5\left(\frac{c_4 L^2}{m_\x m_\y}\right).
\end{align*}
Thus, (\ref{equ:rhsscost}) also holds for $k=2$, provided that $C_2\ge c_4$ and $C_1\ge c_3$.

\paragraph{Total cost.} By combining (\ref{equ:cgcost}) and (\ref{equ:rhsscost}), we can see that the cost (i.e. number of matrix-vector products) of RHSS($k$) per iteration is
\begin{align*}
\left(4C_1^{k+2}+c_1\right)\ln^{k+3}\left(\frac{\max\{c_2,C_2\} L^2}{m_\x m_\y}\right)\sqrt{\frac{L_{\x\y}^2}{m_\x m_\y}\left(\frac{m_\y}{L_{\x\y}}\right)^{\frac{2}{k}} + \frac{L_\x}{m_\x}\left(\frac{m_\y}{L_{\x\y}}\right)^{\frac{1}{k}}}.
\end{align*}
Let us choose $C_2>\max\{c_2,8\}$ and $C_1>\max\{c_1,20\}$. Then, in order to ensure that $\Vert \z_T-\z^*\Vert \le \epsilon$, the number of matrix-vector products that RHSS($k$) needs is
\begin{align*}
&4\left(\frac{L_{\x\y}}{m_\y}\right)^{1/k}\ln\left(\frac{\Vert \z_0-\z^*\Vert}{\epsilon}\right) \cdot\left(4C_1^{k+2}+c_1\right)\ln^{k+3}\left(\frac{C_2 L^2}{m_\x m_\y}\right)\sqrt{\frac{L_{\x\y}^2}{m_\x m_\y}\left(\frac{m_\y}{L_{\x\y}}\right)^{\frac{2}{k}} + \frac{L_\x}{m_\x}\left(\frac{m_\y}{L_{\x\y}}\right)^{\frac{1}{k}}}\\
\le&20C_1^{k+2}\ln\left(\frac{\Vert \z_0-\z^*\Vert}{\epsilon}\right)\ln^{k+3}\left(\frac{C_2 L^2}{m_\x m_\y}\right)\sqrt{\frac{L_{\x\y}^2}{m_\x m_\y}\left(\frac{m_\y}{L_{\x\y}}\right) + \frac{L_\x}{m_\x}\left(\frac{L_{\x\y}}{m_\y}\right)^{\frac{1}{k}}}\\
\le & \sqrt{\frac{L_{\x\y}^2}{m_\x m_\y}+\left(\frac{L_\x}{m_\x}+\frac{L_\y}{m_\y}\right)\left(1+\left(\frac{L_{\x\y}}{m_\y}\right)^{1/k}\right)}\cdot \left(C_1\ln\left(\frac{C_2 L^2}{m_\x m_\y}\right)\right)^{k+3}\ln\left(\frac{\Vert \z_0-\z^*\Vert}{\epsilon}\right).
\end{align*}

\end{proof}

We now discuss how to choose the optimal $k$. Observe that
\begin{align*}
(\ref{equ:rhssbound2})&\le \sqrt{\frac{L_{\x\y}^2}{m_\x m_\y}+\frac{L_\x}{m_\x}+\frac{L_\y}{m_\y}}\ln\left(\frac{\Vert \z_0-\z^*\Vert}{\epsilon}\right)\cdot\underbrace{\left(\frac{L^2}{m_\x m_\y}\right)^{\frac{1}{2k}}\left(C_1\ln\left(\frac{C_2L^2}{m_\x m_\y}\right)\right)^{k+3}}_{(a)}.
\end{align*}
Compared to the lower bound, there is only one additional factor $(a)$, whose logarithm is
\begin{align*}
\ln\left(\left(\frac{L^2}{m_\x m_\y}\right)^{\frac{1}{2k}}C_1^{k+3}\ln^{k+3}\left(\frac{C_2L^2}{m_\x m_\y}\right)\right)=\frac{1}{2k}\ln\left(\frac{L^2}{m_\x m_\y}\right)+(k+3)\ln\left(C_1\ln\left(\frac{C_2 L^2}{m_\x m_\y}\right)\right),
\end{align*}
which is minimized when $k=\sqrt{\frac{\ln\left(\frac{L^2}{m_\x m_\y}\right)}{2\ln\left(C_1\ln\left(\frac{L^2}{m_\x m_\y}\right)\right)}}$, and the minimum value is
\begin{equation*}
3\ln\left(C_1\ln\left(\frac{C_2 L^2}{m_\x m_\y}\right)\right)+\sqrt{\frac{1}{2}\ln\left(\frac{L^2}{m_\x m_\y}\right)\ln\left(C_1\ln\left(\frac{C_2L^2}{m_\x m_\y}\right)\right)}=o\left(\ln\left(\frac{L^2}{m_\x m_\y}\right)\right).
\end{equation*}
I.e. $(a)$ is sub-polynomial in $\frac{L^2}{m_\x m_\y}$. This proves Corollary~\ref{cor:rhss} which states that, when $k=\Theta\left(\sqrt{\ln\left(\frac{L^2}{m_\x m_\y}\right)/\ln\ln\left(\frac{L^2}{m_\x m_\y}\right)}\right)$, the number of matrix vector products that RHSS($k$) needs to find $\z_T$ such that $\Vert\z_T-\z^*\Vert\le\epsilon$ is
\begin{equation*}
\sqrt{\frac{L_{\x\y}^2}{m_\x m_\y}+\frac{L_\x}{m_\x}+\frac{L_\y}{m_\y}}\ln\left(\frac{\Vert \z_0-\z^*\Vert}{\epsilon}\right)\cdot\left(\frac{L^2}{m_\x m_\y}\right)^{o(1)}.
\end{equation*}

\end{document}